\documentclass[11pt]{article}

\usepackage{amsmath,amsfonts,bm}

\def\Figref#1{Figure~\ref{#1}}

\def\Secref#1{Section~\ref{#1}}

\def\Apxref#1{Appendix~\ref{#1}}
\def\twosecrefs#1#2{sections \ref{#1} and \ref{#2}}

\def\eqref#1{equation~\ref{#1}}
\def\Eqref#1{Equation~\ref{#1}}

\def\Thmref#1{Theorem~\ref{#1}}

\def\Lemref#1{Lemma~\ref{#1}}

\def\1{\bm{1}}

\def\vzero{{\bm{0}}}

\def\vb{{\bm{b}}}

\def\ve{{\bm{e}}}

\def\vv{{\bm{v}}}
\def\vw{{\bm{w}}}
\def\vx{{\bm{x}}}
\def\vy{{\bm{y}}}

\def\mA{{\bm{A}}}
\def\mB{{\bm{B}}}
\def\mC{{\bm{C}}}

\def\mM{{\bm{M}}}

\def\mP{{\bm{P}}}

\def\mS{{\bm{S}}}

\def\mU{{\bm{U}}}
\def\mV{{\bm{V}}}
\def\mW{{\bm{W}}}
\def\mX{{\bm{X}}}

\DeclareMathAlphabet{\mathsfit}{\encodingdefault}{\sfdefault}{m}{sl}
\SetMathAlphabet{\mathsfit}{bold}{\encodingdefault}{\sfdefault}{bx}{n}

\def\gE{{\mathcal{E}}}

\def\gG{{\mathcal{G}}}

\def\gL{{\mathcal{L}}}

\def\gN{{\mathcal{N}}}
\def\gO{{\mathcal{O}}}

\def\gV{{\mathcal{V}}}

\def\gZ{{\mathcal{Z}}}

\newcommand{\E}{\mathbb{E}}

\newcommand{\R}{\mathbb{R}}

\usepackage{microtype}
\usepackage{graphicx}
\usepackage{subfigure}
\usepackage{booktabs}

\usepackage{amsthm}
\usepackage{mathtools}

\allowdisplaybreaks

\newtheorem{theorem}{Theorem}[section]

\newtheorem{lemma}[theorem]{Lemma}
\newtheorem{lemmano}{}{}

\newtheorem{prop}[theorem]{Proposition}

\newtheorem{definition}[theorem]{Definition}

\newcommand{\poly}{\mathrm{poly}}

\newcommand{\diag}{\mat{diag}}

\newcommand{\tr}{\mathrm{tr}}

\def\vw{\mathbf{w}}

\newcommand{\mat}[1]{#1}

\newcommand{\ex}{\mathop{\mathbb{E}}\limits}

\newcommand{\revise}[1]{\textcolor{red}{#1}}

\newenvironment{itemize*}{\begin{itemize}[leftmargin=*,topsep=0pt]\setlength{\itemsep}{0pt}\setlength{\parskip}{0pt}}{\end{itemize}}
\newenvironment{enumerate*}{\begin{enumerate}[leftmargin=*,topsep=0pt]\setlength{\itemsep}{0pt}\setlength{\parskip}{0pt}}{\end{enumerate}} \usepackage{amssymb}

\usepackage{hyperref}
\hypersetup{colorlinks=true,citecolor=blue,linkcolor=blue}
\usepackage[parfill]{parskip}
\usepackage{natbib}
\usepackage[margin=0.9in]{geometry}
\usepackage{authblk}

\usepackage{bbm,enumitem,nicefrac}
\newcommand{\gd}{\texttt{GD}}
\newcommand{\gdalg}{\texttt{GD-Alg}}
\newcommand{\gdt}{\texttt{GD2}}
\newcommand{\gdstep}{\gd_{step}^{\eta,t_0}}
\newcommand{\gdreg}{\gd_{reg}^{\lambda}}
\newcommand{\gdpop}{\gd_{pop}}
\newcommand{\gdtworeg}{\gdt^{\lambda}_{reg}}
\newcommand{\reptile}{\texttt{Reptile}}
\newcommand{\replearn}{\texttt{RepLearn}}
\newcommand{\metalearn}{\texttt{Meta}}
\newcommand{\suml}{\sum\limits}
\newcommand{\prodl}{\prod\limits}
\newcommand{\wf}{\mA}
\newcommand{\ws}{\vw}
\newcommand{\Ws}{\mW}
\newcommand{\wo}{{\vw_*}}
\newcommand{\Wo}{{\mW_*}}
\newcommand{\bwo}{{\bar{\vw}_*}}
\newcommand{\Alg}{{\texttt{Alg}}}

\begin{document}

\title{A Sample Complexity Separation between Non-Convex and Convex Meta-Learning}
\date{}
\author[1]{{\large Nikunj Saunshi}}
\author[1]{\large Yi Zhang}
\author[2]{\large Mikhail Khodak}
\author[1,3]{\large Sanjeev Arora}

\affil[1]{\small Department of Computer Science, Princeton University}
\affil[ ]{\texttt {\{nsaunshi, y.zhang, arora\}@cs.princeton.edu}}
\affil[2]{\small School of Computer Science, Carnegie Mellon University}
\affil[ ]{\texttt {khodak@cmu.edu}}
\affil[3]{\small Institute for Advanced Study}
\maketitle

\begin{abstract}
One popular trend in meta-learning is to learn from many training tasks a common initialization for a gradient-based method that can be used to solve a new task with few samples.
The theory of meta-learning is still in its early stages, with several recent learning-theoretic analyses of methods such as Reptile \citep{nichol:18} being for {\em convex models}.
This work shows that convex-case analysis might be insufficient to understand the success of meta-learning, and that even for non-convex models it is important to look inside the optimization black-box, specifically at properties of the optimization trajectory.
We construct a simple meta-learning instance that captures the problem of one-dimensional subspace learning.
For the convex formulation of linear regression on this instance, we show that the new task sample complexity of any {\em initialization-based meta-learning} algorithm is $\Omega(d)$, where $d$ is the input dimension.
In contrast, for the non-convex formulation of a two layer linear network on the same instance, we show that both Reptile and multi-task representation learning can have new task sample complexity of $\gO(1)$, demonstrating a separation from convex meta-learning.
Crucially, analyses of the training dynamics of these methods reveal that they can meta-learn the correct subspace onto which the data should be projected. \end{abstract}

\section{Introduction}
\label{sec:intro}

We consider the problem of {\em meta-learning}, or {\em learning-to-learn} \citep{thrun:98}, in which the goal is to use the data from numerous training tasks to reduce the sample complexity of an unseen but related test task.
Although there is a long history of successful methods in meta-learning and the related areas of multi-task and lifelong learning \citep{evgeniou:04,ruvolo:13}, recent approaches have been developed with the diversity and scale of modern applications in mind.
This has given rise to simple, model-agnostic methods that focus on learning a good initialization for some gradient-based method such as stochastic gradient descent (SGD), to be run on samples from a new task \citep{finn:17,nichol:18}.
These methods have found widespread applications in a variety of areas such as computer vision \citep{nichol:18}, reinforcement learning \citep{finn:17}, and federated learning \citep{mcmahan:17}.

Inspired by their popularity, several recent learning-theoretic analyses of meta-learning have followed suit, eschewing customization to specific hypothesis classes such as halfspaces \citep{maurer:13,balcan:15} and instead favoring the convex-case study of gradient-based algorithms that could potentially be applied to deep neural networks \citep{denevi:19a,khodak:19b}.
This has yielded results showing that meta-learning an {\em initialization} by using methods similar to Reptile \citep{nichol:18} for convex models leads to a reduction in sample complexity of unseen tasks.
These benefits are shown using natural notions of task-similarity like the average distance between the risk minimizers of tasks drawn from an underlying meta-distribution.
A good initialization in these models is one that is close to the population risk minimizers for tasks in this meta-distribution.

In this paper we argue that, even in some simple settings, such convex-case analyses are insufficient to understand the success of initialization-based meta-learning algorithms. For this purpose, we pose a simple instance for meta-learning linear regressors that share a one-dimensional subspace, for which we prove a sample complexity separation between convex and non-convex methods. 
Specifically, our contributions are the following:
\begin{itemize}[leftmargin=*]
	\item We show, in the convex formulation of linear regression on this instance, a new task sample complexity lower bound of $\Omega(d)$ for {\em any initialization-based meta-learning algorithm}. This suggests that 
	\emph{no} amount of meta-training data can yield an initialization that can be used by a common gradient-based within-task algorithms to solve a new task with fewer samples than if no meta-learning had been done; thus initialization-based meta-learning in the convex formulation fails to learn the underlying task-similarity.
	\item We show for the same instance that formulating the model as a two-layer linear network -- an over-parameterization of the same hypothesis class -- allows a Reptile-like procedure to use training tasks from this meta-learning instance and find an initialization for gradient descent that will have $\gO(1)$ sample complexity on a new task. To the best of our knowledge, this is the first sample complexity analysis of initialization-based meta-learning algorithms in the non-convex setting.
	\item Central to our proof is a trajectory-based analysis to analyze properties of the solution found by a specific procedures like Reptile or gradient descent on a representation learning objective. For the latter, we show that looking at the trajectory is crucial as not all minimizers can learn the subspace structure.
	\item Finally, we revisit existing upper bounds for the convex case. We show that our lower bound does not contradict these upper bounds, since their task similarity measure of average parameter distance is large in our case. We complement this observation by proving that the existing bounds are tight, in some sense, and going beyond them will require additional structural assumptions.
\end{itemize}
\paragraph{Paper organization:}
We discuss related work in \Secref{sec:related}.
\Secref{sec:meta} sets up notation for the rest of the paper, formalizes initialization-based meta-learning methods and defines the subspace meta-learning instance that we are interested in.
The lower bound for linear regression is stated in \Secref{sec:lower_bound}, while the corresponding upper bounds for non-convex meta-learning with two-layer linear network is provided in \Secref{sec:upper_bound}.
While all proofs are provided in the appendix, we give a sketch of the proofs for the upper bounds in \Secref{sec:proof_sketch} to highlight the key steps in the trajectory-based analysis and discuss why such an analysis is important.
A discussion about tightness of existing convex-case upper bounds can be found in \Secref{sec:gvlower}.
 
\section{Related Work}
\label{sec:related}

There is a rich history of theoretical analysis of learning-to-learn \citep{baxter:00,maurer:05,maurer:16}.
Our focus is on a well-studied setting in which tasks such as halfspace learning share a common low-dimensional subspace, with the goal of obtaining sample complexity depending on this sparse structure rather than on the ambient dimension \citep{maurer:09,maurer:13,balcan:15,denevi:18a,bullins:19,khodak:19b}.
While these works derive specialized algorithms, we instead focus on learning an initialization for gradient-based methods such as SGD or few steps of gradient descent \citep{finn:17,nichol:18}.
Some of these methods have recently been studied in the convex setting \citep{denevi:19a,khodak:19b,zhou:19}.
Our results show that such convex-case analyses cannot hope to show adaptation to an underlying low-dimensional subspace leading to dimension-independent sample complexity bounds.
On the other hand, we show that their guarantees using distance-from-initialization are almost tight for the meta-learning of convex Lipschitz functions.

To get around the limitations of convexity for the problem of meta-learning a shared subspace, we instead study non-convex models.
While the optimization properties of gradient-based meta-learning algorithms have been recently studied in the non-convex setting \citep{fallah:19,rajeswaran:19,zhou:19}, these results only provide stationary-point convergence guarantees and do not show a reduction in sample complexity, the primary goal of meta-learning.
Our theory is more closely related to recent empirical work that tries to understand various inherently non-convex properties of learning-to-learn.
Most notably, \citet{arnold:19} hypothesize and show some experimental evidence that the success of gradient-based meta-learning requires non-convexity, a view theoretically supported by our work.
Meanwhile, \citet{raghu:19} demonstrate that the success of the popular MAML algorithm \citep{finn:17} is likely due to its ability to learn good data-representations rather than adapt quickly;
in this work our subspace meta-learning guarantees are for a representation learning algorithm that only adapts the last layer at test-time.

Our results draw upon work motivated by understanding deep learning that analyzes trajectories and implicit regularization in deep linear neural networks \citep{saxe:14,gunasekar:18,saxe:19,gidel:19}. 
The analysis of solutions found by gradient flow in deep linear networks by \citep{saxe:14,gidel:19} form a core component of our analysis.
In this vein, \citet{lampinen:19} recently studied the dynamics of deep linear networks in the context of transfer learning and show that jointly learning linear representations using two tasks will yield smaller error on each one than individual task learning.
However their guarantees are not for an unseen task drawn from a distribution, but only for two given tasks, and crucially not for gradient-based meta-learning methods. 
\section{Meta-Learning Setup}
\label{sec:meta}
\subsection{Notations}\label{subsec:notations}
Let $[N]$ denote the set $\{1,\dots,N\}$.
We use $\vx$ for vectors, $\mM$ for matrices, $I_d$ for $d$ dimensional identity matrix and $\vzero_d$ for the all-zero vector in $d$ dimensions.
$\|\cdot\|$ is used to denote the $\ell_2$ norm.
For a function $\ell:X\times Y\rightarrow Z$, we use $\ell(x,\cdot):Y\rightarrow Z$ to denote a function of the second argument when the first argument is set to $x$.
For a finite set $S$, $x\sim S$ denotes sampling uniformly from $S$.
We also need the ReLU function $[x]_+=x\mathbbm{1}\{x\ge0\}$.
For a sequence $\{a_1,\dots,a_T\}$, we use $a_{i:j}$ for $j\ge i$ to denote the set $\{a_i,\dots,a_j\}$.

\subsection{Task distribution and excess risk}\label{subsec:tasks}
We are interested in regression tasks of the following form
\begin{align}\label{eq:task_loss}
\ell_{\rho}(\theta):=\ex_{(\vx,y)\sim\rho}(f(\vx, \theta)-y)^2
\end{align}
where we abuse notation and use $\rho$ to denote a task as well as its associated data distribution.
The input $\vx$ is a vector in $\R^d$ and $y$ is real-valued scalar.
The function $f:\R^d\times\Theta \rightarrow \R$ is a regressor of choice, e.g. a linear function or a deep neural network, that is parametrized by $\theta\in\Theta$.
Often one only has access to samples $S=\{(x_i,y_i)\}_{i=1}^{n}$ from the unknown distribution $\rho$, and the empirical risk is defined as
\begin{align}\label{eq:emp_loss}
	\ell_S(\theta) = \ex_{(\vx,y)\sim S}(f(\vx,\theta)-y)^2
\end{align}

While various formalizations for meta-learning exist, we present one that is most convenient for the presentation of this work.
In our meta-learning setting, we assume that there is an underlying unknown distribution $\mu$ over tasks.
Given access to a $T$ training tasks $\rho_1,\dots,\rho_T$ sampled from $\mu$, the goal of a meta-learner $\metalearn$ is to learn some underlying structure that relates the tasks in $\mu$ and output a within-task algorithm $\Alg=\metalearn(\rho_{1:T})$ that can be used to solve a new task sampled from $\mu$.
To solve a new task $\rho\sim\mu$ by using training set $S$ from $\rho$, the meta-learned algorithm $\Alg$ outputs parameters $\Alg(S)\in\Theta$.
The average risk of an algorithm that uses $n$ samples from a new task is
\begin{align*}
	\gL_n(\Alg,\mu)=\ex_{\rho\sim\mu}\ex_{S\sim\rho^n}\ell_\rho(\Alg(S))
\end{align*}
We define the excess risk of $\Alg$ as $\gE_n(\Alg,\mu)=\gL_n(\Alg,\mu)-\gL^*(\mu)$, where $\gL^*(\mu)=\ex_{\rho\sim\mu}\inf\limits_{\theta\in\Theta}\ell_\rho(\theta)$ is the minimum achievable risk by the class $\Theta$ with complete knowledge of the distribution $\mu$.

\subsection{Initialization-based meta-learning}\label{subsec:gd_based}
We focus on a popular approach in meta-learning that uses training tasks to learn an {\em initialization} of the model parameters.
This initialization is fed into a pre-specified {\em gradient-based} algorithm that updates model parameters starting from this initialization by using samples from a new task.
We refer to these methods as {\em initialization-based} meta-learning methods and they are restricted to return within-task algorithms of the form $\Alg(\cdot)=\gdalg(\cdot;\theta_{init})$, where $\gdalg$ runs some gradient-based algorithm starting from the initialization $\theta_{init}\in\Theta$ on an objective function that depends on the input training set $S$.
For example, we can denote the algorithm of gradient descent as $\gd(S;\theta_{init})$, that runs gradient descent to convergence on the empirical risk $\ell_S$ by starting from the initialization $\theta_{init}$.
The definitions of the various initialization-based meta-learning and within-task algorithms that we analyze are in \twosecrefs{subsec:lower_within_task}{subsec:upper_within_task}.
In the subsequent sections, we will concretely define the distribution of tasks $\mu$ and the meta-learning algorithms we are interested in.

\subsection{Meta-learning a subspace}\label{subsec:subspace}
For meta-learning to be meaningful, the tasks must share some common structure.
Here we focus on a structure that assumes the existence of a low-dimensional representation of the data that suffices to solve all the tasks, specifically, a linear representation.
To capture this idea, we construct a simple but instructive meta-learning instance.

We are interested in tasks $\rho_{\vw}$ for $\vw\in\R^d$, where the distribution is defined as follows
\begin{align}\label{eq:distribution}
(\vx,y)\sim\rho_{\vw}: \vx\sim\gN(0,I_d), y\sim\gN(\vw^{\top}\vx,\sigma^2)
\end{align}
The target $y$ for $\vx$ is a linear function of $\vx$ plus a zero-mean Gaussian noise\footnote{We can extend all results to $y=\vw^{\top}\vx+\xi$, where $\xi$ is independent of $\vx$, just has 0 mean and variance $\sigma^2$.} added to it.
A {\em meta-learning instance} $\mu_{\wo}$ is defined as uniform distribution over two tasks $\rho_{\wo}$ and $\rho_{-\wo}$ for a fixed but unknown vector $\wo\in\R^d$.
Note that for every point $\vx\in\R^d$, only the projection of $\vx$ onto the direction of $\wo$ is necessary to solve all tasks in $\mu_\wo$.
Thus the hope is that a meta-learning algorithm picks up on this structure and learns to project data onto this subspace for sample efficiency on a new task.
The average task risk and excess risk for an algorithm $\Alg$ can then be written as
\begin{align}\label{eq:excess_risk_w}
	\gL_n(\Alg, \mu_{\wo}) &= \ex_{s\sim\{\pm1\}}\ex_{S\sim\rho_{s\wo}^n}\ell_{s\wo}(\Alg(S))\nonumber\\
	\gE_n(\Alg, \mu_{\wo}) &= \gL_n(\Alg, \mu_{\wo})-\gL^*(\mu_{\wo})
\end{align}
In the subsequent sections, we describe the convex setting of linear regression and the equally expressive non-convex setting of a two-layer linear network regressor.
Our main result shows that while no meta-learning algorithm can learn a meaningful initialization for a gradient-based within-task algorithm in the convex setting, standard meta-learning algorithms like Reptile on a two-layer linear network can in fact learn to project the data on the one-dimensional subspace and thus reduce the sample complexity for a new task from $\Omega(d)$ to $\gO(1)$.

\section{Convex Meta-Learning Lower Bound}
\label{sec:lower_bound}
In this section, we use a regression function $f$ that is linear in $\vx$ to solve the meta-learning instance $\mu_\wo$.
We have $\Theta=\R^d$, the parameters are $\theta=\vw, \vw\in\R^d$ and the regressor is $f(\vx, \vw)\coloneqq \vw^{\top}\vx$.
Using the definition of the distribution in \Eqref{eq:distribution}, for $s\in\{\pm1\}$ we get
\begin{align}\label{eq:linear_loss_closed_form}
	\ell_{s\wo}(\vw) = \ex_{(\vx,y)\sim\rho_{s\wo}}(\vw^{\top}\vx-y)^2 = \|\vw-s\wo\|^2+\sigma^2
\end{align}

Thus we have
$\gL^*(\mu_{\wo})=\ex_{s\sim\{\pm1\}}\inf\limits_{\vw\in\R^d}\ell_{s\wo}(\vw)=\sigma^2$.

\subsection{Within-task algorithms}\label{subsec:lower_within_task}
As described in \Secref{subsec:tasks}, we consider within-task algorithms that are based on gradient descent.
A meta-learner is allowed to learn an initialization $\vw_0\in\R^d$ that is used as a starting point to run a gradient-based algorithm on a new task.
We will show lower bounds for the following algorithms

$\gdstep(S;\vw_0)$ - {\bf GD for $t_0$ steps}:\\
Runs gradient descent with learning rate $\eta$ for $t_0$ steps on $\ell_S$ (defined in \Eqref{eq:emp_loss}). Starting from $\vw_0$, follow the dynamics below and return $\vw_{t_0}$.
\begin{align*}
	\vw_{t+1} = \vw_{t} - \eta \nabla_{\vw}\ell_S(\vw_t)
\end{align*}

$\gdreg(S;\vw_0)$ - {\bf $\lambda$-regularized GD}:\\
Runs gradient descent with vanishingly small learning rate (gradient flow) to convergence on $\ell_{S,\lambda}$
\begin{align}\label{eq:reg_loss}
	\ell_{S,\lambda}(\vw)=\ex_{(\vx,y)\sim S}\left[(\vw^{\top}\vx-y)^2\right] + \frac{\lambda}{2}\|\vw\|^2
\end{align}
Starting from $\vw_0$, follow the dynamics below, return $\vw_{\infty}$.
\begin{align*}
	\frac{d\vw_t}{dt} = -\nabla_{\vw}\ell_{S,\lambda}(\vw_t)
\end{align*}

In the next section we will provide lower bounds on the excess risk for all initialization-based meta-learning algorithms that return initializations for the above algorithms.
Note that some of these algorithms have been used in prior work;
most notably, $\gdstep$ is the base-learner used by MAML \citep{finn:17}, so our convex-case lower-bounds hold directly for any initialization it might learn.

\subsection{Lower bounds}\label{subsec:lower_bounds}
We use the definition of excess risk $\gE_n$ from \Eqref{eq:excess_risk_w} and formally define sample complexity for a meta-learned within-task algorithm below
\begin{definition}[Sample complexity]
The minimum number of samples needed from a new task for a within-task algorithm $\Alg$ to have excess risk smaller than $\epsilon$ is
\begin{align}\label{eq:n_eps}
	n_\epsilon(\Alg,\mu_\wo) = \min\{n\in\mathbb{N} : \gE_n(\Alg,\mu_\wo)\le\epsilon\}
\end{align}
\end{definition}
We will proceed to show a lower bound for all meta-learning algorithms that return an initialization to be used by algorithms $\gdstep$ and $\gdreg$ described in the previous subsection.
We assume that $\|\wo\|=\sigma=r$ to make the noise of the same order as the signal and for simplicity of presentation.
The lower bounds in more generality can be found in \Apxref{asubsec:lower_bounds}.
\begin{theorem}\label{thm:lower_bound_gd}
Suppose $\|\wo\|=\sigma=r$ and $\epsilon\in\left(0,\frac{r^2}{2}\right)$. For every initialization $\vw_0\in\R^d$ that can be learned by an initialization-based meta-learning algorithm, the number of samples needed to have $\epsilon$ excess risk on a new task is
\begin{align*}
\min\limits_{\lambda\ge 0}~~n_\epsilon(\gdreg(\cdot;\vw_0), \mu_\wo) &= \Omega\left(\frac{dr^2}{\epsilon}\right)\\\min\limits_{\eta>0,t_0\in\mathbb{N}_+}n_\epsilon(\gdstep(\cdot;\vw_0), \mu_\wo) &= \Omega\left(\frac{dr^2}{\epsilon}\right)\end{align*}
\end{theorem}

\paragraph{Remark.} We remark the strength of the lower bound for the following reasons:
\begin{itemize}[leftmargin=*]
	\item The bound holds even if the meta learner has seen \emph{infinitely} many tasks sampled from $\mu$ and has access to the \emph{population loss} for each task.
	\item Even regularization techniques like explicit $\ell_2$-regularization or early stopping \emph{cannot} benefit from a meta-learned initialization.
	\item Note that the condition $\epsilon\le \nicefrac{r^2}{2}$ is not restrictive since even a trivial learner that always outputs $\vzero_d$ for every task has error exactly $r^2$.
\end{itemize}

This demonstrates that the convex formulation does not do justice to the practical efficacy of such algorithms.
We provide the proof of this result and even tigher lower bounds in the appendix.
The proofs are based on finding a closed-form expression for the solutions found by $\gdreg$ and $\gdstep$ and showing that, in fact, no initialization has better excess risk than the trivial initialization of $\vzero_d$.

\section{Non-Convex Meta-Learning Upper Bound}
\label{sec:upper_bound}
We now use a two layer linear network as the regressor $f$.
The parameters in this case are $\theta=(\wf,\ws), \wf\in\R^{m\times d}, \ws\in\R^n$.
The regressor $f$ is then defined as $f(\vx,(\wf,\ws))\coloneqq \ws^{\top}\wf \vx$.
As before,
\begin{align}\label{eq:2layer_loss_closed_form}
	\ell_{s\wo}((\wf,\ws)) &= \|\wf^{\top}\ws-s\wo\|^2+\sigma^2
\end{align}
Again it is easy to see that $\gL^*(\mu_\wo)=\sigma^2$.
We now describe the within-task algorithms of interest and the initialization-based meta-algorithms for which we show guarantees.

\subsection{Within-task and meta-learning algorithms}\label{subsec:upper_within_task}
We are interested in the following within-task algorithms.

$\gdpop(\rho;(\wf_0, \ws_0))$ - {\bf Population GD}:\\
Runs gradient descent with vanishingly small learning rate (gradient flow) to convergence on $\ell_\rho$.
Starting from $(\wf_0, \ws_0)$, follow the dynamics below, return $(\wf_{\infty}, \ws_{\infty})$.
\begin{align*}
	\frac{d\wf_{t}}{dt} = -\nabla_{\wf} \ell_{\rho}((\wf_t,\ws_t));~ \frac{d\ws_{t}}{dt} = -\nabla_{\ws} \ell_{\rho}((\wf_t,\ws_t))
\end{align*}

$\gdtworeg(S;(\wf_0, \ws_0))$ - {\bf Second-layer regularized GD}:\\
Runs gradient descent with tiny learning rate (gradient flow) to convergence on $\ell_{S,\lambda}(\ws;\wf_0)$
\begin{align}\label{eq:reg_loss_2layer}
	\ell_{S,\lambda}(\ws;\wf_0)=\ex_{(\vx,y)\sim S}(\ws^{\top}\wf_0\vx-y)^2 + \frac{\lambda}{2}\|\ws\|^2
\end{align}
Starting from $\ws_0$, follow the dynamics below by only updating $\ws$, return $(\wf_0,\ws_{\infty})$
\begin{align*}
	\frac{d\ws_{t}}{dt} = -\nabla_{\ws} \ell_{S,\lambda}(\ws_t;\wf_0))
\end{align*}

We will be showing guarantees for initializations learned by two meta-learning algorithms, \reptile~ and \replearn.
A meta-learner receives $T$ training tasks $\{\rho_1,\dots,\rho_T\}$ sampled independently from $\mu_{\wo}$; each task is either $\rho_{\wo}$ or $\rho_{-\wo}$.
For simplicity of analysis, we assume that the learner has access to the population losses for these tasks, since we are mainly concerned about the new task sample complexity.
While simplistic, showing guarantees even in this setting requires a non-trivial analysis.
Note that the lower bound for linear regression holds even with access to population loss function for any number of training tasks.
The first meta-learning algorithm of interest is the following

$\reptile(\rho_{1:T}, (\wf_0,\ws_0))$ - {\bf Reptile:}\\
Starting from $(\wf_{0},\ws_{0})$, the initialization maintained by the algorithm is sequentially updated as $(\wf_{i+1},\ws_{i+1})=(1-\tau)(\wf_{i},\ws_{i}) + \tau\gdpop(\ell_{\rho_{i+1}}, (\wf_{i},\ws_{i}))$ for some $0<\tau<1$.
At the end of $T$ tasks, return $\wf_{T}$.

On encountering a new task, \reptile~slowly interpolates between the current initialization and the solution for the new task obtained by running gradient descent on it starting from the current initialization.
As mentioned earlier, this method has enjoyed empirical success \citep{mcmahan:17}.
The second algorithm of interest is reminiscent to multi-task representation learning.

$\replearn(\rho_{1:T}, (\wf_0,\ws_{0,1:T}))$ - {\bf Representation learning:}
Starting from $(\wf_0,\ws_{0,1:T})$, run gradient flow on the following objective function: $\gL_{rep}(\wf,\ws_{1:T}) = \frac{1}{T}\suml_{i=1}^{T} \ell_{\rho_i}(\wf,\ws_{i})$, return $\wf_\infty$ at the end.
\begin{align*}
	\frac{d\wf_t}{dt}=-\nabla_{\wf} \gL_{rep}(\wf_t,\ws_{t,1:T});
	\frac{d\ws_{t,i}}{dt}=-\nabla_{\ws_{t,i}} \gL_{rep}(\wf_t,\ws_{t,i}), i\in[T]
\end{align*}
This is a standard objective for multi-task representation learning used in prior work, occasionally equipped with a regularization term for $\ws_{1:T}$.
For our analysis we do not need an explicit regularizer, just like \citet{saxe:14} and \citet{gidel:19}.

\subsection{Upper bounds}\label{subsec:upper_bounds}
Recall that $\gE_{n}(\gdtworeg(\cdot;(\wf,\vzero_d)),\mu_\wo)$ is the excess risk for the initialization $\wf$ that is used by $\gdtworeg$.
We will show that with access to a feasible number of training tasks, both \reptile~ and \replearn~ can learn an initialization with small $\gE_{n}$.
We first prove the upper bounds for \reptile ~under the assumption that $\|\wo\|=\sigma=r$.
\begin{theorem}\label{thm:upper_bound_reptile}
Starting with $(\wf_0,\ws_0)=(\kappa I_d,\vzero_d)$, let $\wf_T=\reptile(\rho_{1:T},(\wf_0,\ws_0))$ be the initialization learned from $T$ tasks $\{\rho_{1},\dots,\rho_{T}\}\sim_{i.i.d.} \mu_{\wo}^T$. If $T\ge poly(d,r,1/\epsilon,\log(1/\delta),\kappa)$ and $\tau=\gO(T^{-1/3})$, then with probability at least $1-\delta$ over sampling of $T$ tasks,
\begin{align*}
	\min\limits_{\lambda\ge0}~\gE_{n}(\gdtworeg(\cdot;(\wf_T,\vzero_d)), \mu_{\wo}) \le \epsilon + \frac{cr^2}{n}
\end{align*}
for a small constant $c$. Thus with the same probability, we have
\begin{align*}
	\min\limits_{\lambda\ge 0}~n_\epsilon&(\gdtworeg(\cdot;\wf_T,\vzero_d), \mu_\wo) = \gO\left(\frac{r^2}{\epsilon}\right)
\end{align*}
\end{theorem}
The proof can be found in \Apxref{asubsec:main_results}.
Thus we can show that a standard meta-learning method like $\reptile$ can learn a useful initialization for a gradient-based within-task algorithm like $\gdtworeg$.
A sketch of the proof in \Secref{sec:proof_sketch} will demonstrate that the $\reptile$ update surprisingly amplifies the component along $\wo$ in the spectrum of the first layer $\wf$, while keeping the components orthogonal to $\wo$ unchanged.
Interestingly, even though both $\wo$ and $-\wo$ appear as tasks, the meta-initialization $\kappa>0$ ensures that they do not cancel each other out in the first layer, unlike in the second layer.
In contrast to the convex-case lower bound, we only need $\gO(r^2/\epsilon)$ samples for a new task, thus showing gap of $d$ between convex and non-convex meta-learning in our setting.
We now show a similar result for \replearn~under the assumption of $\|\wo\|=\sigma=r$.
\begin{theorem}\label{thm:upper_bound_replearn}
With $(\wf_{0},\ws_{0,1:T})=(\kappa I_d,\vzero_d,\dots,\vzero_d)$, let $\wf_T=\replearn(\rho_{1:T},(\wf_{0},\ws_{0,1:T})),$ be the initialization learned using $T$ tasks $\{\rho_{1},\dots,\rho_{T}\}\sim_{i.i.d.} \mu_{\wo}^T$. If $T\ge poly(d,r,1/\epsilon,\log(1/\delta),\kappa)$, then with probability at least $1-\delta$ over sampling of the $T$ tasks,
\begin{align*}
	\min\limits_{\lambda\ge0}~\gE_{n}(\gdtworeg(\cdot;(\wf_T,\vzero_d)), \mu_{\wo}) \le \epsilon + \frac{cr^2}{n}
\end{align*}
for a small constant $c$. Thus with the same probability, we have
\begin{align*}
	\min\limits_{\lambda\ge 0}~n_\epsilon&(\gdtworeg(\cdot;\wf_T,\vzero_d), \mu_\wo) = \gO\left(\frac{r^2}{\epsilon}\right)
\end{align*}
\end{theorem}
Yet again we can show a new task sample complexity of $\gO(r^2/\epsilon)$.
We now sketch the proofs of the upper bounds to highlight the interesting parts of the proof and to show the need for a trajectory-based analysis. 
\section{Proof Sketch}
\label{sec:proof_sketch}

We first present a proof sketch for the guarantees provided for the \reptile~ algorithm in \Thmref{thm:upper_bound_reptile} and for \replearn~ in \Thmref{thm:upper_bound_replearn}.
Following that we will present an argument for why a trajectory-based analysis is necessary, by looking more closely at the representation learning objective.

\subsection{\reptile~ sketch}\label{subsec:reptile_sketch}
For simplicity assume $\|\wo\|=1$.
Let the $T$ training tasks be $\rho_1,\dots,\rho_T$, where $\rho_i=\rho_{s_i\wo}$ for $s_i$ is uniformly sampled from $\{\pm1\}$.
Recall the update: $(\wf_{i+1},\ws_{i+1})=(1-\tau)(\wf_{i},\ws_{i}) + \tau\gdpop(\ell_{\rho_{i+1}}, (\wf_{i},\ws_{i}))$.
The proof involves showing the following key properties of the dynamics of $\gdtworeg$ and the interpolation updates:

{\bf Step 1:} Starting from $\wf_0=\kappa I_d, \ws=\vzero_d$, the initialization learned by the meta-learning algorithm always satisfies $\wf_i=(a_i-\kappa)\wo\wo^{\top}+\kappa I_d$, $\ws_i=b_i\wo$.

Thus the updates by \reptile~ ensure that $\wf$ is only updated in the direction of $\wo\wo^{\top}$ and $\ws$ is updated in the direction of $\wo$.
This is proved by induction, where the crucial step is to show that if at time $i$ we start with $\wf_i,\ws_i$ that satisfy the above condition, then interpolating towards the output of $\gdpop$ still maintains this condition.
Step 2 below shows exactly this and, in fact, we can get the exact dynamics for the sequence $\{a_i,b_i\}$.

{\bf Step 2:} Initialized with $\wf=(a-\kappa)\wo\wo^{\top}+\kappa I_d,\ws=b\wo$ for $a>b\ge0$, the solution found by $\gdpop$ is $\bar{\wf},\bar{\ws}=\gdpop(\rho_{s\wo},(\wf,\ws))$ where $\bar{\wf}=(\bar{a}-\kappa)\wo\wo^{\top}+\kappa I_d$, $\bar{\ws}=\bar{b}\wo$, for $\bar{a}=f(a,b,s)$ and $\bar{b}=g(a,b,s)$
\begin{align*}
	f(a,b,s) &= \sqrt{\frac{(a^2-b^2)+\sqrt{4+(a^2-b^2)^2}}{2}}\\
	g(a,b,s) &= s\sqrt{\frac{-(a^2-b^2)+\sqrt{4+(a^2-b^2)^2}}{2}}
\end{align*}
This, along with step 1, gives us the dynamics of $a_i,b_i$
\begin{equation}\label{eq:reptile_dynamics}
\begin{aligned}
	a_{i+1}&=a_{i} + \tau(f(a_i,b_i,s_{i+1}) - a_i)\\
	b_{i+1}&=b_{i} + \tau(g(a_i,b_i,s_{i+1}) - b_i)
\end{aligned}
\end{equation}
This is the step where we use the analysis of the trajectory of gradient flow on two-layer linear networks that was done first in \citet{saxe:14} and later made robust in \citet{gidel:19}.
While their focus was on the case where the two layers are initialized at exactly the same scale, we need to analyze the case where $\wf$ and $\ws$ are initialized differently; this was analyzed in the appendix of \citet{saxe:14}.
In fact, as we will see in step 3, having $\kappa\neq 0$ when $\ws_0=\vzero_d$ is crucial in showing that $\wf$ can learn the subspace.
Refer to \Figref{fig:dynamics} for more insights into the dynamics induced by $f$ and $g$.

{\bf Step 3:} We show a very important property satisfied by the dynamics of $a_i,b_i$ described in \Eqref{eq:reptile_dynamics}: $a_i$ is an increasing sequence.
Since the sequence $s_{1:T}$ is a random sequence in $\{\pm1\}^T$, $a_T$ and $b_T$ are random variables.
However even though $s_{i+1}$ has 0 mean, $s_{i+1}$ only affects the sign of $b_i$ but not $a_i$, as evident in \Eqref{eq:reptile_dynamics}.
In fact, we can show that if initialized with $\kappa>0$, $a_i$ always increases; the same is however not true for $b_i$.
We show that for the meta-initialization of $a_0=\kappa$ and $b_0=0$, with high probability, $a_T=\tilde{\Omega}\left(\min\left\{\frac{1}{2\sqrt{\tau}}, (\tau T)^{1/4}\right\}\right)$.
Picking $\tau=\gO\left(T^{-1/3}\right)$, we get that $a_T=\tilde{\Omega}(T^{1/6})$.
Thus for an appropriate choice of the interpolating parameter $\tau$, $a_T\rightarrow\infty$ as $T\rightarrow\infty$.
So we know that in the limit, $\wf_T$ is basically a rank one matrix in the direction of $\wo\wo^{\top}$.
In the next step we show why such an $\wf_T$ reduces sample complexity.

{\bf Step 4:} To gain intuition for why the learned $\wf_T$ reduces sample complexity, notice that the only information about input $\vx$ that is needed to make predictions for all tasks in $\mu_\wo$ is its projection on $\wo$.
Thus if all data points are projected on $\wo$, we could just learn a 1-dimensional classifier on the projected data. 
So after this projection, the task would be reduced to a 1-dimensional regression problem that has a sample complexity of $\gO(1/\epsilon)$.
With $\wf_T=(a_T-\kappa)\wo\wo^{\top}+\kappa I_d$, we are learning a classifier for a new task on the linearly transformed data $\wf_T\vx$ instead.
For large enough $T$, $a_T$ is large enough that $\wf_T$ almost acts like a projection onto the subspace of $\wo$, thus leading to a reduction in sample complexity from $\Omega(d/\epsilon)$ to $\gO(1/\epsilon)$.

\begin{figure*}[t!]\label{fig:dynamics}
  \centering
  \subfigure[One step update]{\includegraphics[scale=0.55]{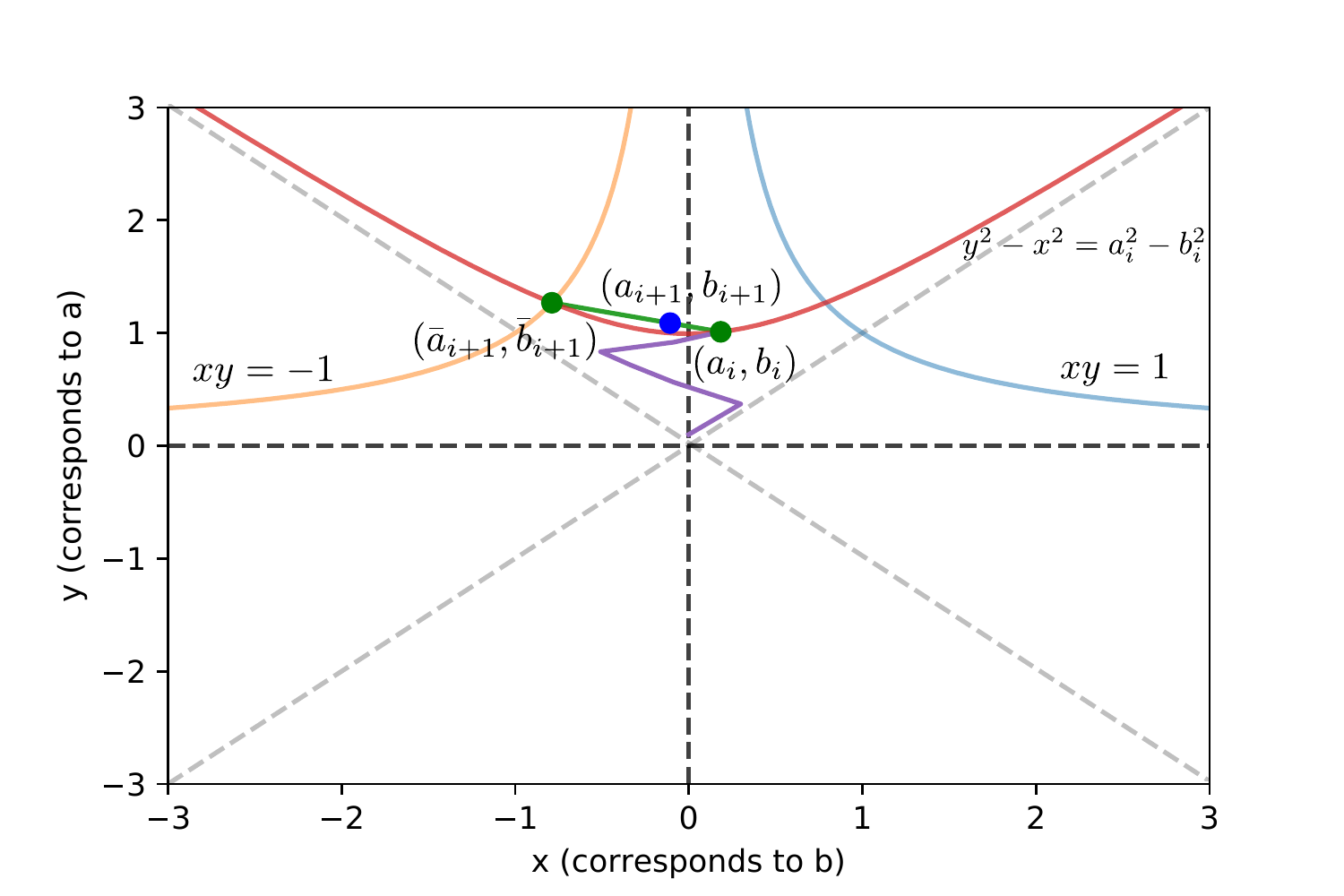}}
  \subfigure[Evolution for 1000 steps]{\includegraphics[scale=0.55]{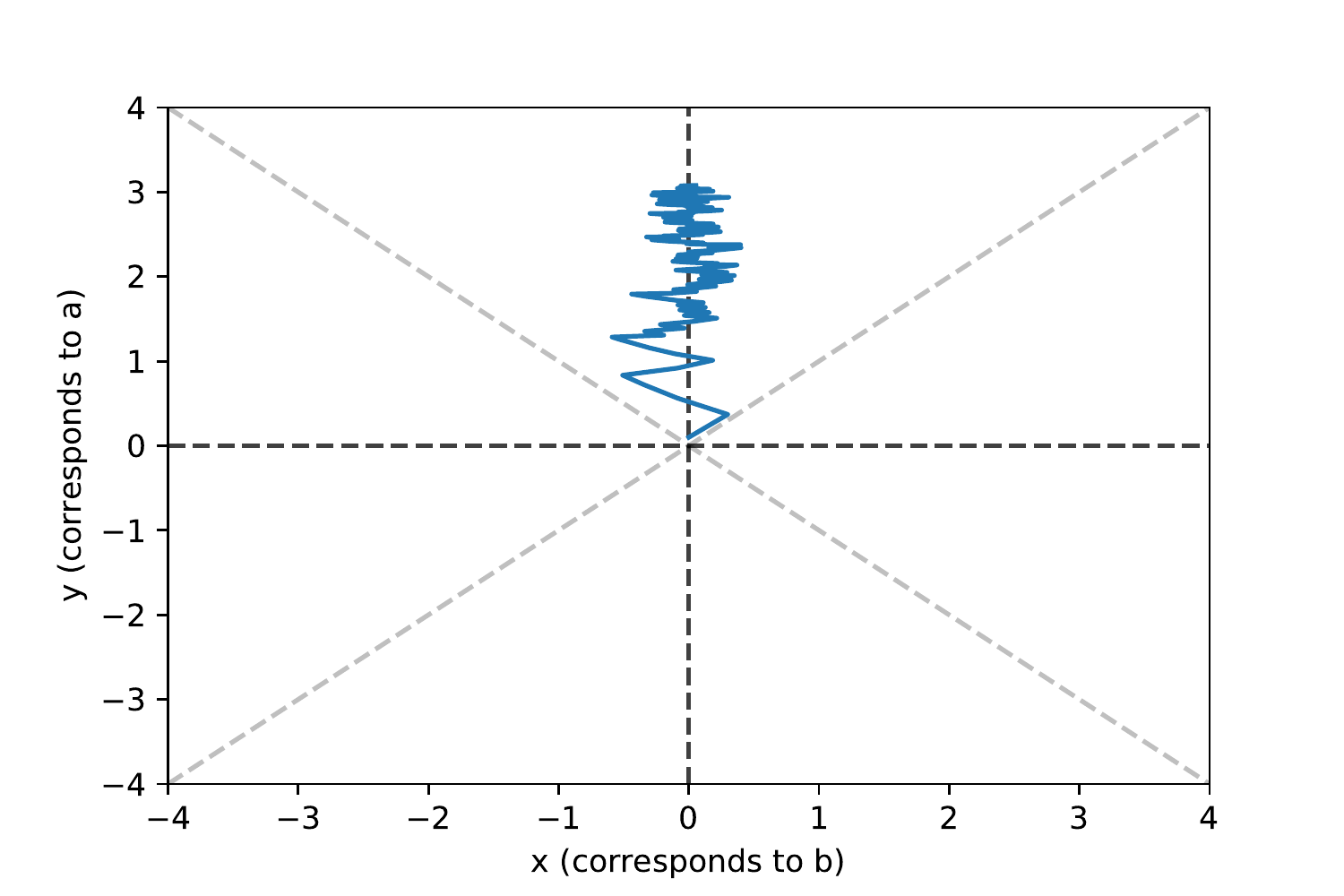}}
  \caption{The two figures correspond to a run with $T=1000$ tasks, $\tau=0.3$, $(a_0,b_0)=(0.1,0)$ and $\|\wo\|=1$. The first figure shows what the updates from \Eqref{eq:reptile_dynamics} looks like at step $i$ when $s_{i+1}=-1$. It can be shown that $\bar{a}_{i+1}=f(a_i,b_i,s_{i+1})$ and $\bar{b}_{i+1}=g(a_i,b_i,s_{i+1})$ always satisfy $\bar{a}_{i+1}^2-\bar{b}_{i+1}=a_i^2-b_i^2$, thus the solution $(\bar{a}_{i+1},\bar{b}_{i+1})$ will be the intersection of the curves $xy=s_{i+1}=-1$ and $y^2-x^2=a_i^2-b_i^2$ and $(a_{i+1},b_{i+1})$ is the appropriate interpolation. The second figure shows the entire dynamics of $(a_i,b_i)$ for the same setting. As evident, $a_i$ is always increasing while $b_i$ fluctuates around its mean value of 0.}
\end{figure*}

\subsection{\replearn~ sketch}\label{subsec:replearn_sketch}
Recall that the representation learning algorithm runs gradient descent on $\gL_{rep}$ by starting from $(\kappa I_d,\vzero_d,\dots,\vzero_d)$, where
\begin{align*}
	\gL_{rep}(\wf,\ws_{1},\dots,\ws_{T}) = \frac{1}{T}\suml_{i=1}^{\top} \ell_{\rho_i}(\wf,\ws_{i})
	=\frac{1}{T}\suml_{i=1}^T \|\wf^{\top}\ws_i-s_i\wo\|^2 = \frac{1}{T}\|\wf^{\top}\Ws-\Wo\|^2
\end{align*}
where $\Ws\in\R^{d\times T}$ has $\ws_i$ as its $i^{th}$ column and $\Wo\in\R^{d\times T}$ has $s_i\wo$ as its $i^{th}$ column.
This objective is a special case of the deep linear regression objective studied in \citet{saxe:14,gidel:19}, except with an unbalanced initialization for $\wf$ and $\Ws$.
Using a very similar analysis technique, one can show that gradient flow on this objective will converge to $\wf_\infty=(a_\infty-\kappa)\wo\wo^{\top}+\kappa I_d$, where for a sufficiently small $\kappa$, $a_\infty=\Omega(T^{1/4})$.
Just like the previous section, the first layer has learned the subspace and will reduce sample complexity of a new task to $\gO(1/\epsilon)$.

\subsection{Why trajectory is important}\label{subsec:trajectory}
As evident in the proof sketches above, we relied heavily on analyzing the specific trajectory of different methods, whether it is for gradient descent on a specific objective function or the interpolation updates in \reptile.
A natural question is whether simple analysis techniques that only look at properties of {\em all minimizers} of some objective function can lead to similar conclusions.
We answer this question for the representation learning objective negatively.
In particular, we construct a minimizer of the objective $\gL_{rep}$ where the first layer does not learn any structure about the subspace and will have $\Omega(d/\epsilon)$ new task sample complexity.
This bad minimizer is very simple: $\wf=I_d, \ws_i=s_i\wo, \forall i\in[T]$.
While the existence of such a solution is not too surprising, it does illustrate that analyzing the dynamics of the specific algorithms used might be as important as the objective functions themselves. 
\section{Tightness of Existing Bounds}
\label{sec:gvlower}

In providing a first non-convex sample complexity analysis of gradient-based meta-learning, our results have also exposed a fundamental limitation of convex methods:
in the presence of very natural subspace structure they are unable to learn an initialization that exploits it to obtain a good sample complexity.
There is thus a tension between this result and recent upper bounds that use other intuitive assumptions on the task-distribution to show reduced sample complexity of similar or identical methods \citep{denevi:19a,khodak:19b,zhou:19}.
Broadly, these results show that gradient-based meta-learning methods can adapt to a similarity measure that depends on the closeness of minimizing parameters for the tasks.
For convex models they obtain upper bounds on the excess risk of form
\begin{equation}\label{eq:gvupper}
\gE_n(\Alg,\mu)=\gO\left(\frac{GV}{\sqrt n}\right)
\end{equation}
for large enough number of training tasks $T$, where $V^2=\min_{\phi\in\Theta}\E_{\rho\sim\mu}\|\phi-\operatorname{Proj}_{\Theta_\rho^\ast}(\phi)\|^2$ is the average variation of the optimal task parameters, for $\Theta_\rho^\ast=\arg\min_{\theta\in\Theta}\ell_\rho(\theta)$, and $G$ is the Lipschitz constant with respect to the Euclidean norm.

These results, however, do not contradict our convex-case lower bounds in \Secref{sec:lower_bound} because our tasks are not similar in the same sense.
While the parameters lie on a subspace, the average variation of optimal parameters remains large.
However, while the distance-based task-similarity measure is natural and intuitive, we believe that a low-dimensional representation structure such as ours may be more explanatory for the success of gradient-based meta-learning algorithms.
In fact the importance of representation learning in the success of popular gradient-based methods has been shown by existing empirical results \citep{raghu:19}.

Additionally we argue that existing upper bounds may not be very meaningful in the context of current practical applications.
The term $GV$ in Equation~\ref{eq:gvupper} can be lower bounded by Jensen's inequality
\begin{align*}
	GV
	=G\sqrt{\min_{\phi\in\Theta}\E_{\rho\sim\mu}\|\phi-\operatorname{Proj}_{\Theta_\rho^\ast}(\phi)\|^2}
	\ge \min_{\phi\in\Theta}G\E_{\rho\sim\mu}\|\phi-\operatorname{Proj}_{\Theta_\rho^\ast}(\phi)\|
	\ge\min_{\phi\in\Theta}\ex_{\rho\sim\mu}[\ell_\rho(\phi)-\ell_\rho^*]
\end{align*}
where $\min_{\phi\in\Theta}\E_{\rho\sim\mu}\ell_\rho(\phi)$ is the minimum achievable risk by a {\em single common parameter} for all tasks from the class $\Theta$ and $\ell_\rho^*=\min_{\theta\in\Theta}\ell_\rho(\theta)$.
In common meta-learning settings, the average risk $\E_{\rho\sim\mu}\ell_\rho(\phi)$ of any fixed parameter $\phi\in\Theta$ is large, e.g. due to label-shuffling in tasks like Omniglot \citep{lake:11} and Mini-ImageNet \citep{ravi:17} or due to symmetry in the tasks around zero like in the sine wave task \citep{finn:17}.

Given the above drawbacks, it is natural to ask if this bound of $GV$ can be improved in the convex settings prior work considers.
We answer this negatively.
Below we adapt an information-theoretic argument from \citet{agarwal:12} to show that such a dependence is unavoidable when analyzing a distance-based task-similarity notion for convex $G$-Lipschitz functions, and thus that existing results are almost tight:
\begin{theorem}\label{thm:gvlower}
	For any $G,V>0$, there exists a domain $\gZ$, parameter class $\Theta\subseteq\R^d$ and a distribution $\mu$ over tasks such every $\rho\sim\mu$ is a distribution over $\gZ$ and $\ell_{\rho}(\theta)=\E_{z\sim \rho}\ell_z(\theta)$ where $\ell_z:\Theta\rightarrow\R$ is convex and $G$-Lipschitz  w.r.t. the Euclidean norm for every $z\in\gZ$. Additionally, $\Theta$ satisfies
	$$\min_{\phi\in\Theta}\E_{\rho\sim\mu}\|\phi-\operatorname{Proj}_{\Theta_\rho^\ast}(\phi)\|\le V$$
	and
	$$\gE_n(\Alg, \mu)=\Omega\left(GV\min\left\{\frac1{\sqrt n},\frac1{\sqrt d}\right\}\right)$$
	for any algorithm $\Alg:\gZ^n\rightarrow\Theta$ that returns a parameter given a training set.
\end{theorem}
A consequence of this theorem is that without additional assumptions other than convexity, Lipschitzness and small average parameter variation, one cannot hope to improve upon existing bounds.
This, coupled with the fact that the existing bounds can be large in practical settings, makes a case for the need for more structural assumptions and a shift to non-convexity for analyses of meta-learning.

\section{Conclusions and Future Work}
\label{sec:conclusion}

In this work we look at a family of initialization-based meta-learning methods that has enjoyed empirical success.
Using a simple meta-learning problem of linear predictors in a 1-dimensional subspace, we show a gap in the new task sample complexity between meta-learning using linear regression and meta-learning using two-layer linear networks.
This is, to our knowledge, is the first non-convex sample complexity analysis of initialization-based meta-learning, and there are many interesting future directions to be pursued.
\begin{itemize}[leftmargin=*]
	\item $k$-subspace learning: while the lower bound for the convex setting trivially holds if the task predictions came from a $k$-dimensional subspace for $k>1$, showing that an algorithm like Reptile can have sample complexity of $\gO(k)$ is an open problem. While this can be proved for the representation learning objective using a very similar analysis, showing it for Reptile, which only learns one second layer instead of $T$ unlike the representation learning objective, might require stronger tools. There is experimental evidence suggesting that such a statement might be true.
	\item Weaker distributional assumptions: while showing upper bounds was non-trivial under current assumptions, one would hope to show guarantees under weaker and more realistic assumptions, such as a more general data distribution, different input distributions across tasks, and access to only finitely many samples from training tasks.
	\item One common bottleneck for the above points is a robust analysis of the dynamics of linear networks when the initializations are not appropriately aligned. While \citet{gidel:19} provide a perturbation analysis for this, $\epsilon$-perturbation at the initialization scales as $\epsilon e^{ct^2}$ in the final solution, where $t$ is the time for which gradient descent/flow is run. It would be nice to have an analysis with a more conservative error propagation, perhaps exploiting structured perturbations.
	\item Deep neural network: while analysis for linear networks can be a first cut to understanding non-convex meta-learning, it would be interesting to see if the insights gained from this setting are useful for the more interesting setting of non-linear neural networks.
\end{itemize} 
\bibliography{references}

\begin{thebibliography}{30}
\providecommand{\natexlab}[1]{#1}
\providecommand{\url}[1]{\texttt{#1}}
\expandafter\ifx\csname urlstyle\endcsname\relax
  \providecommand{\doi}[1]{doi: #1}\else
  \providecommand{\doi}{doi: \begingroup \urlstyle{rm}\Url}\fi

\bibitem[Agarwal et~al.(2012)Agarwal, Bartlett, Ravikumar, and
  Wainwright]{agarwal:12}
Alekh Agarwal, Peter~L. Bartlett, Pradeep Ravikumar, and Martin~J. Wainwright.
\newblock Information-theoretic lower bounds on the oracle complexity of
  stochastic convex optimization.
\newblock \emph{IEEE Transactions on Information Theory}, 58\penalty0
  (5):\penalty0 3235--3249, 2012.

\bibitem[Arnold et~al.(2019)Arnold, Iqbal, and Sha]{arnold:19}
S\'{e}bastien M.~R. Arnold, Shariq Iqbal, and Fei Sha.
\newblock Decoupling adaptation from modeling with meta-optimizers for meta
  learning.
\newblock arXiv, 2019.

\bibitem[Balcan et~al.(2015)Balcan, Blum, and Vempala]{balcan:15}
Maria-Florina Balcan, Avrim Blum, and Santosh Vempala.
\newblock Efficient representations for lifelong learning and autoencoding.
\newblock In \emph{Proceedings of the 28th Annual Conference on Learning
  Theory}, 2015.

\bibitem[Baxter(2000)]{baxter:00}
Jonathan Baxter.
\newblock A model of inductive bias learning.
\newblock \emph{Journal of Artificial Intelligence Research}, 12:\penalty0
  149--198, 2000.

\bibitem[Bullins et~al.(2019)Bullins, Hazan, Kalai, and Livni]{bullins:19}
Brian Bullins, Elad Hazan, Adam Kalai, and Roi Livni.
\newblock Generalize across tasks: Efficient algorithms for linear
  representation learning.
\newblock In \emph{Proceedings of the 30th International Conference on
  Algorithmic Learning Theory}, 2019.

\bibitem[Davis and Kahan(1970)]{davis1970rotation}
Chandler Davis and William~Morton Kahan.
\newblock The rotation of eigenvectors by a perturbation. iii.
\newblock \emph{SIAM Journal on Numerical Analysis}, 7\penalty0 (1):\penalty0
  1--46, 1970.

\bibitem[Denevi et~al.(2018)Denevi, Ciliberto, Stamos, and Pontil]{denevi:18a}
Giulia Denevi, Carlo Ciliberto, Dimitris Stamos, and Massimiliano Pontil.
\newblock Incremental learning-to-learn with statistical guarantees.
\newblock In \emph{Proceedings of the Conference on Uncertainty in Artificial
  Intelligence}, 2018.

\bibitem[Denevi et~al.(2019)Denevi, Ciliberto, Grazzi, and Pontil]{denevi:19a}
Giulia Denevi, Carlo Ciliberto, Riccardo Grazzi, and Massimiliano Pontil.
\newblock Learning-to-learn stochastic gradient descent with biased
  regularization.
\newblock In \emph{Proceedings of the 36th International Conference on Machine
  Learning}, 2019.

\bibitem[Evgeniou and Pontil(2004)]{evgeniou:04}
Theodoros Evgeniou and Massimiliano Pontil.
\newblock Regularized multi-task learning.
\newblock In \emph{Proceedings of the 10th ACM SIGKDD International Conference
  on Knowledge Discovery and Data Mining}, 2004.

\bibitem[Fallah et~al.(2019)Fallah, Mokhtari, and Ozdaglar]{fallah:19}
Alireza Fallah, Aryan Mokhtari, and Asuman Ozdaglar.
\newblock On the convergence theory of gradient-based model-agnostic
  meta-learning algorithms.
\newblock arXiv, 2019.

\bibitem[Finn et~al.(2017)Finn, Abbeel, and Levine]{finn:17}
Chelsea Finn, Pieter Abbeel, and Sergey Levine.
\newblock Model-agnostic meta-learning for fast adaptation of deep networks.
\newblock In \emph{Proceedings of the 34th International Conference on Machine
  Learning}, 2017.

\bibitem[Gidel et~al.(2019)Gidel, Bach, and Lacoste-Julien]{gidel:19}
Gauthier Gidel, Francis Bach, and Simon Lacoste-Julien.
\newblock Implicit regularization of discrete gradient dynamics in deep linear
  neural networks.
\newblock \emph{Advances in Neural Information Processing Systems}, 2019.

\bibitem[Gunasekar et~al.(2018)Gunasekar, Lee, Soudry, and
  Srebro]{gunasekar:18}
Suriya Gunasekar, Jason Lee, Daniel Soudry, and Nathan Srebro.
\newblock Characterizing implicit bias in terms of optimization geometry.
\newblock In \emph{Proceedings of the 35th International Conference on Machine
  Learning}, 2018.

\bibitem[Khodak et~al.(2019)Khodak, Balcan, and Talwalkar]{khodak:19b}
Mikhail Khodak, Maria-Florina Balcan, and Ameet Talwalkar.
\newblock Adaptive gradient-based meta-learning methods.
\newblock In \emph{Advances in Neural Information Processing Systems}, 2019.

\bibitem[Lake et~al.(2017)Lake, Salakhutdinov, Gross, and Tenenbaum]{lake:11}
Brenden~M. Lake, Ruslan Salakhutdinov, Jason Gross, and Joshua~B. Tenenbaum.
\newblock One shot learning of simple visual concepts.
\newblock In \emph{Proceedings of the Conference of the Cognitive Science
  Society (CogSci)}, 2017.

\bibitem[Lampinen and Ganguli(2019)]{lampinen:19}
Andrew Lampinen and Surya Ganguli.
\newblock An analytic theory of generalization dynamics and transfer learning
  in deep linear networks, 2019.

\bibitem[Maurer(2005)]{maurer:05}
Andreas Maurer.
\newblock Algorithmic stability and meta-learning.
\newblock \emph{Journal of Machine Learning Research}, 6:\penalty0 967--994,
  2005.

\bibitem[Maurer(2009)]{maurer:09}
Andreas Maurer.
\newblock Transfer bounds for linear feature learning.
\newblock \emph{Machine Learning}, 2009.

\bibitem[Maurer and Pontil(2013)]{maurer:13}
Andreas Maurer and Massimiliano Pontil.
\newblock Excess risk bounds for multitask learning with trace norm
  regularization.
\newblock In \emph{Proceedings of the 26th Annual Conference on Learning
  Theory}, 2013.

\bibitem[Maurer et~al.(2016)Maurer, Pontil, and Romera-Paredes]{maurer:16}
Andreas Maurer, Massimiliano Pontil, and Bernardino Romera-Paredes.
\newblock The benefit of multitask representation learning.
\newblock \emph{Journal of Machine Learning Research}, 17\penalty0
  (1):\penalty0 2853--2884, 2016.

\bibitem[McMahan et~al.(2017)McMahan, Moore, Ramage, Hampson, and {Aguera y
  Arcas}]{mcmahan:17}
H.~Brendan McMahan, Eider Moore, Daniel Ramage, Seth Hampson, and Blaise
  {Aguera y Arcas}.
\newblock Communication-efficient learning of deep networks from decentralized
  data.
\newblock In \emph{Proceedings of the 20th International Conference on
  Artifical Intelligence and Statistics}, 2017.

\bibitem[Nichol et~al.(2018)Nichol, Achiam, and Schulman]{nichol:18}
Alex Nichol, Joshua Achiam, and John Schulman.
\newblock On first-order meta-learning algorithms.
\newblock arXiv, 2018.

\bibitem[Raghu et~al.(2019)Raghu, Raghu, Bengio, and Vinyals]{raghu:19}
Aniruddh Raghu, Maithra Raghu, Samy Bengio, and Oriol Vinyals.
\newblock Rapid learning or feature reuse? {T}owards understanding the
  effectiveness of {MAML}.
\newblock arXiv, 2019.

\bibitem[Rajeswaran et~al.(2019)Rajeswaran, Finn, Kakade, and
  Levine]{rajeswaran:19}
Aravind Rajeswaran, Chelsea Finn, Sham~M. Kakade, and Sergey Levine.
\newblock Meta-learning with implicit gradients.
\newblock In \emph{Advances in Neural Information Processing Systems}, 2019.

\bibitem[Ravi and Larochelle(2017)]{ravi:17}
Sachin Ravi and Hugo Larochelle.
\newblock Optimization as a model for few-shot learning.
\newblock In \emph{Proceedings of the 5th International Conference on Learning
  Representations}, 2017.

\bibitem[Ruvolo and Eaton(2013)]{ruvolo:13}
Paul Ruvolo and Eric Eaton.
\newblock {ELLA}: An efficient lifelong learning algorithm.
\newblock In \emph{Proceedings of the 30th International Conference on Machine
  Learning}, 2013.

\bibitem[Saxe et~al.(2014)Saxe, Mcclelland, and Ganguli]{saxe:14}
Andrew~M. Saxe, James~L. Mcclelland, and Surya Ganguli.
\newblock Exact solutions to the nonlinear dynamics of learning in deep linear
  neural network.
\newblock In \emph{In International Conference on Learning Representations},
  2014.

\bibitem[Saxe et~al.(2019)Saxe, McClelland, and Ganguli]{saxe:19}
Andrew~M. Saxe, James~L. McClelland, and Surya Ganguli.
\newblock A mathematical theory of semantic development in deep neural
  networks.
\newblock \emph{Proceedings of the National Academy of Sciences}, 116\penalty0
  (23):\penalty0 11537--11546, 2019.

\bibitem[Thrun and Pratt(1998)]{thrun:98}
Sebastian Thrun and Lorien Pratt.
\newblock \emph{Learning to Learn}.
\newblock Springer Science \& Business Media, 1998.

\bibitem[Zhou et~al.(2019)Zhou, Yuan, Xu, Yan, and Feng]{zhou:19}
Pan Zhou, Xiaotong Yuan, Huan Xu, Shuicheng Yan, and Jiashi Feng.
\newblock Efficient meta learning via minibatch proximal update.
\newblock In \emph{Advances in Neural Information Processing Systems}, 2019.

\end{thebibliography}
\bibliographystyle{plainnat}
\newpage
\appendix
\onecolumn

\section{Appendix Overview}\label{asec:overview}
Appendix is organized as follows:

In \Apxref{asec:convex} we prove the lower bounds for convex meta-learning.
\begin{itemize}
	\item \Apxref{asubsec:lower_bounds} has the proofs for stronger versions of the main lower bound result, \Thmref{thm:lower_bound_gd}, that shows lower bounds for the within-task methods of $\gdreg$ and $\gdstep$ (MAML).
	\item \Apxref{asubsec:closed_form} has proofs for closed form solutions found by $\gdreg$ and $\gdstep$ given $n$ samples for a new task. These results are useful to prove the aforementioned theorems.
	\item \Apxref{asubsec:other_proofs} contains proofs for auxiliary lemmas.
\end{itemize}

In \Apxref{asec:non_convex} we prove the upper bounds for non-convex meta-learning.
\begin{itemize}
	\item \Apxref{asubsec:reptile} formalizes the steps mentioned in the proof sketch for $\reptile$ from \Secref{subsec:reptile_sketch}. It has the bulk of the proofs about the dynamics of gradient-based algorithms.
	\item \Apxref{asubsec:main_results} proves the main upper bound theorems, \Thmref{thm:upper_bound_reptile} and \Thmref{thm:upper_bound_replearn}.
\end{itemize}

In \Apxref{asec:gv_lower} we prove the tightness of current distance-based convex meta-learning lower bounds (\Thmref{thm:gvlower}).

\section{Convex proofs}\label{asec:convex}

\subsection{Lower bounds}\label{asubsec:lower_bounds}
Before proving the lower bounds, we present the following lemma about the closed form solutions found by $\gdreg$ and $\gdstep$ starting from an initialization $\vw_0$; the proof of this can be found in \Apxref{asubsec:closed_form}.

Note that every $S=\{(\vx_i,y_i)\}\sim\rho_{\vv}^n$ is unique determined by a matrix $\mX\in\R^{n\times d}$ and a noise vector $\xi\in\R^n$, where the $i^{th}$ row of $\mX$ is $\vx_i$ and $\xi_i=y_i-\vv^{\top}\vx_i$ are i.i.d. samples from $\gN(0,\sigma^2)$.

\begin{lemma}\label{lem:gd_closed_form}
Let $S=(\mX,\xi)$ be a sample from $\rho_{\vv}$, $\mX\in\R^{n\times d}, \xi\in\R^n$. Let $\Sigma_\mX=\frac{1}{n}\suml_{i=1}^{n}\vx_i\vx_i^{\top}\in\R^{d\times d}$
\begin{align*}
	\gdreg(S; \vw_{0}) &= (I_d-(\Sigma_\mX+\lambda I_d)^{\dagger}(\Sigma_\mX+\lambda I_d))\vw_{0} + (\Sigma_\mX+\lambda I_d)^{\dagger}\Sigma_\mX\vv + \frac{1}{n}(\Sigma_\mX+\lambda I_d)^{\dagger}\mX^{\top}\xi
\end{align*}
\begin{align*}
	\gdstep(S; \vw_{0}) &= (I_d-\eta\Sigma_\mX)^{t_0}\vw_{0} + (I_d-(I_d-\eta\Sigma_\mX)^{t_0})\Sigma_\mX^{\dagger}\Sigma_\mX\vv + \frac{1}{n}(I_d-(I_d-\eta\Sigma_\mX)^{t_0})\Sigma_\mX^{\dagger}\mX^{\top}\xi
\end{align*}
\end{lemma}
Here we use $\mA^{\dagger}$ to denote the Moore-Penrose pseudo-inverse of matrix $\mA$.
Note that while the inverse exists for all $\lambda>0$, we use the pseudo-inverse for $\lambda=0$.
Also since $\mX$ and $\xi$ do not depend on $\vv$, the only dependence of $\gdreg(S,\vw_0)$ and $\gdstep(S,\vw_0)$ on $\vv$ is the second term in each of the equations.
Since the solutions of both $\gdreg$ and $\gdstep$ are linear in $\vw_0$, $\vv$ and $\xi$, the following lemma will be useful; the proof can be found in \Apxref{asubsec:other_proofs}.
\begin{lemma}\label{lem:excess_risk_linear}
For $S=(\mX,\xi)$ sampled from $\rho_{\vv}$, $\mX\in\R^{n\times d}, \xi\in\R^n$, if $\Alg(S;\vw_0)=\mA_\mX\vw_0 + \mB_\mX\vv + \mC_\mX\xi$, then
\begin{align*}
	\gE_n(\Alg(\cdot; \vw_{0}),\mu_\wo) \ge \underbrace{\ex_{\mX}[\|(I_d-\mB_\mX)\wo\|^2]}_{bias(\Alg)}+\underbrace{\ex_{\mX}[\sigma^2\tr(\mC_\mX^{\top}\mC_\mX)]}_{var(\Alg)}
\end{align*}
where $bias(\Alg)$ is the error in predicting $\wo$ and $var(\Alg)$ is error due to noise in labels in the training data $S$.
\end{lemma}
Note that the lower bound on excess risk does not depend on the initialization $\vw_0$ or the matrix $\mA_\mX$.
We are now ready to prove the following stronger version of \Thmref{thm:lower_bound_gd}.

\paragraph{Proving \Thmref{thm:lower_bound_gd}:} We now prove strengthened versions of the theorem for $\gdreg$ and $\gdstep$ separately.
From Definition~\ref{eq:n_eps}, we have $n_\epsilon(\Alg,\mu_\wo) = \min\limits_{n\in\mathbb{N}} : \gE_n(\Alg,\mu_\wo)\le\epsilon$ is the minimum number of samples needed to achieve excess risk at most $\epsilon$.
\begin{lemmano}[{\bf \Thmref{thm:lower_bound_gd}(a)}]
For every $\vw_0\in\R^d$, number of samples needed to have $\epsilon$ excess risk on a new task is
\begin{align*}
	\gE_n(\gdreg(\cdot;\vw_0),\mu_\wo) \ge \begin{cases}
		\frac{d\|\wo\|^2\sigma^2}{\|\wo\|^2n+\sigma^2d} & \text{if $n\ge d$}\\\\
		\frac{n}{d}\frac{\|\wo\|^2\sigma^2}{\|\wo\|^2+\sigma^2} + \frac{(d-n)}{d}\|\wo\|^2& \text{if $n<d$}
	\end{cases}
\end{align*}
Furthermore if $\|\wo\|=\sigma=r\ge 1$ and $\epsilon\in\left(0,\frac{r^2}{2}\right)$ , then the number of samples needed to achieve excess error of $\epsilon$ is
\begin{align*}
	\min\limits_{\lambda\ge 0}~~n_\epsilon(\gdreg(\cdot;\vw_0), \mu_\wo) \ge \frac{dr^2}{2\epsilon}
\end{align*}
\end{lemmano}
\begin{proof}[Proof of \Thmref{thm:lower_bound_gd}(a)]
Consider $m$ samples $S$ from $\rho_{s\wo}$, where $s\in\{\pm1\}$.
As observed earlier, sampling $S\sim\rho_{s\wo}$ corresponds to sampling $\mX\sim\gN^n(0,I_d)$ and $\xi\sim\gN(0,\sigma^2I_n)$.
From \Lemref{lem:gd_closed_form}, we get 
$$\gdreg(S; \vw_{0}) = (I_d-(\Sigma_\mX+\lambda I_d)^{\dagger}(\Sigma_\mX+\lambda I_d))\vw_{0} + (\Sigma_\mX+\lambda I_d)^{\dagger}\Sigma_\mX\vv + \frac{1}{n}(\Sigma_\mX+\lambda I_d)^{\dagger}\mX^{\top}\xi$$
Instantiating \Lemref{lem:excess_risk_linear} with $\mA_\mX=(I_d-(\Sigma_\mX+\lambda I_d)^{\dagger}(\Sigma_\mX+\lambda I_d))$, $\mB_\mX=(\Sigma_\mX+\lambda I_d)^{\dagger}\Sigma_\mX$ and $\mC_\mX=\frac{1}{n}(\Sigma_\mX+\lambda I_d)^{\dagger}\mX^{\top}$, we get 
\begin{align*}
	\gE_n(\gdreg(\cdot; \vw_{0}),\mu_\wo)
	&=\ex_{\mX}\left[\|(I_d-(\Sigma_\mX+\lambda I_d)^{\dagger}\Sigma_\mX)\wo\|^2\right]+\ex_{\mX}\left[\sigma^2\tr\left(\frac{\mX}{n}{(\Sigma_\mX+\lambda I_d)^{\dagger}}^2\frac{\mX^{\top}}{n}\right)\right]\\
	&=^{(a)}\ex_{\mX}\left[\|(I_d-(\Sigma_\mX+\lambda I_d)^{\dagger}\Sigma_\mX)\wo\|^2\right]+\frac{\sigma^2}{n^2}\ex_{\mX}\left[\tr({(\Sigma_\mX+\lambda I_d)^{\dagger}}^2\mX^{\top}\mX)\right]\\
	&=\underbrace{\ex_{\mX}\left[\|(I_d-(\Sigma_\mX+\lambda I_d)^{\dagger}\Sigma_\mX)\wo\|^2\right]}_{bias(\gdreg)}+\underbrace{\frac{\sigma^2}{n}\ex_{\mX}\left[\tr({(\Sigma_\mX+\lambda I_d)^{\dagger}}^2\Sigma_\mX)\right]}_{var(\gdreg)}
\end{align*}
where $(a)$ follows from property about trace that $\tr(AB)=\tr(BA)$ and the definition of $\Sigma_\mX$.
We now lower bound the bias and variance terms separately.
Let $\Sigma_\mX=\mV\mS\mV^{\top}$ be the full SVD, where $\mS=diag(s_1,\dots,s_d)$ is a diagonal matrix such that $s_1\ge s_2\ge\dots\ge s_d\ge0$.
Let $\vv_i$ be the $i^{th}$ column of $\mV$. Note that $\mV^{\top}\mV=\mV\mV^{\top}=I_d$.
\paragraph{Bias:} The bias term can be handled by first noticing the following
\begin{align*}
	bias(\gdreg)
	&= \wo^{\top}(I_d-(\Sigma_\mX+\lambda I_d)^{\dagger}\Sigma_\mX)^2\wo
	= \wo^{\top}(\mV\mV^{\top}-(\mV\mS\mV^{\top}+\lambda \mV\mV^{\top})^{\dagger}\mV\mS\mV^{\top})^2\wo\\
	&= \wo^{\top}(\mV\mV^{\top}-\mV(\mS+\lambda I_d)^{\dagger}\mV^{\top}\mV\mS\mV^{\top})^2\wo
	= \wo^{\top}\mV(I_d-(\mS+\lambda I_d)^{\dagger}\mS)^2\mV^{\top}\wo\\
	&= \suml_{i=1}^{d} h(s_i,\lambda)(\wo^{\top}\vv_i)^2,\text{ where}\\
	h(s,\lambda) &= \begin{cases}
		\frac{\lambda^2}{(s+\lambda)^2} &\text{if $s>0$}\\
		1 &\text{if $s=0$}
	\end{cases}
\end{align*}
We can split the expectation w.r.t. $\mX$ into expectation w.r.t. $\mS$ and the conditional expectation of $\mV$ given $\mS$.
A crucial observation is that since the distribution of the rows of $\mX$ is isotropic gaussian, no direction in space is special.
Thus, conditioned on $\mS$, the distribution of $\vv_i$ should be identical for all $i$ and we must have that $\ex_{\mX}[\vv_i | \mS]=0$ and $\ex_{\mX}[\vv_i\vv_i^{\top}|\mS]=CI_d$ for some constant $C$.
The constant $C$ can be calculated by noting that $\|\vv_i\|=1$.
So we get $1=\ex_{\mX}[\vv_i^{\top}\vv_i|\mS] = \tr\left(\ex_{\mX}[\vv_i\vv_i^{\top}|\mS]\right)=C\tr(I_d)=Cd$, giving $C=\frac{1}{d}$.
Then the bias is
\begin{align}\label{eq:bias_lower_bound}
	bias(\gdreg)\nonumber
	&= \ex_{\mS}\left[\ex_{\mV}\left[\suml_{i=1}^{d} h(s_i,\lambda)(\wo^{\top}\vv_i)^2\bigg|\mS\right]\right]
	= \ex_{\mS}\left[\suml_{i=1}^{d} h(s_i,\lambda)\ex_{\vv_i}\left[(\wo^{\top}\vv_i)^2\bigg|\mS\right]\right]\\\nonumber
	&= \ex_{\mS}\left[\suml_{i=1}^{d} h(s_i,\lambda)\ex_{\vv_i}\left[\wo^{\top}\vv_i\vv_i^{\top}\wo\bigg|\mS\right]\right]
	= \ex_{\mS}\left[\suml_{i=1}^{d} h(s_i,\lambda)\wo^{\top}\ex_{\vv_i}\left[\vv_i\vv_i^{\top}\bigg|\mS\right]\wo\right]\\
	&= \ex_{\mS}\left[\suml_{i=1}^{d} h(s_i,\lambda)\frac{\|\wo\|^2}{d}\right]
	= \frac{\|\wo\|^2}{d}\ex_{\mS}\left[\suml_{i=1}^{d} h(s_i,\lambda)\right]
\end{align}

\paragraph{Variance:} We now look at the variance term
\begin{align*}
	var(\gdreg)
	&=\ex_{\mX}\left[\frac{\sigma^2}{n}\tr({(\Sigma_\mX+\lambda I_d)^{\dagger}}^2\Sigma_\mX)\right]
	= \ex_{\mX}\left[\frac{\sigma^2}{n}\tr({(\mV\mS\mV^{\top}+\lambda \mV\mV^{\top})^{\dagger}}^2\mV\mS\mV^{\top})\right]\\
	&= \ex_{\mX}\left[\frac{\sigma^2}{n}\tr(\mV{(\mS+\lambda I_d)^{\dagger}}^2\mV^{\top}\mV\mS\mV^{\top})\right]
	= \ex_{\mX}\left[\frac{\sigma^2}{n}\tr(\mV{(\mS+\lambda I_d)^{\dagger}}^2\mS\mV^{\top})\right]\\
	&= \ex_{\mX}\left[\frac{\sigma^2}{n}\tr({(\mS+\lambda I_d)^{\dagger}}^2\mS\mV^{\top}\mV)\right]
	= \ex_{\mS}\left[\frac{\sigma^2}{n}\tr({(\mS+\lambda I_d)^{\dagger}}^2\mS)\right]\\
	&= \frac{\sigma^2}{n}\ex_{\mS}\left[\suml_{i=1}^dg(s_i,\lambda)\right]\text{, where}\\
	g(s,\lambda) &= \begin{cases}
		\frac{s}{(s+\lambda)^2} & \text{if $s>0$}\\
		0 & \text{if $s_i=0$}
	\end{cases}
\end{align*}
Thus we get the following lower bound for the excess risk
\begin{align*}
	\gE_n(\gdreg(\cdot;\vw_0),\mu_\wo)
	&= bias(\gdreg) + var(\gdreg)
	\ge \frac{\|\wo\|^2}{d}\ex_{\mS}\left[\suml_{i=1}^{d} h(s_i,\lambda)\right] + \frac{\sigma^2}{n}\ex_{\mS}\left[\suml_{i=1}^dg(s_i,\lambda)\right]\\
	&= \ex_{\mS}\left[\suml_{i=1}^{d}\frac{\|\wo\|^2}{d} h(s_i,\lambda) + \frac{\sigma^2}{n}g(s_i,\lambda)\right]
\end{align*}
We will show that $\frac{\|\wo\|^2}{d} h(s,\lambda) + \frac{\sigma^2}{n}g(s,\lambda) \ge \frac{\|\wo\|^2\sigma^2}{\|\wo\|^2ns+\sigma^2d}$.
While this is evident when $s=0$, since $h(0,\lambda)=1$ and $g(0,\lambda)=0$, for $s>0$ the left hand side reduces to $\frac{\|\wo\|^2}{d} \frac{\lambda^2}{(\lambda + s)^2} + \frac{\sigma^2}{ns}\frac{s^2}{(\lambda + s)^2}$.
This is of the form $\alpha a^2 + \beta b^2$ where $\alpha=\frac{\|\wo\|^2}{d}, \beta=\frac{\sigma^2}{ns}$ and $a=\frac{\lambda}{(\lambda + s)},b=\frac{s}{(\lambda + s)}$ satisfy $a+b=1$.
The following simple lemma (proof in \Apxref{asubsec:other_proofs}) will help us prove the desired inequality.
\begin{lemma}\label{lemma:asq_bsq}
For $\alpha,\beta\ge0$, we have
\[
\min\limits_{a,b ~s.t.~a+b=1} \alpha a^2 + \beta b^2 = \frac{\alpha \beta}{\alpha + \beta}
\]
\end{lemma}
Using the above lemma, we get $\gE_n(\gdreg(\cdot;\vw_0),\mu_\wo)\ge \ex_{\mS}\left[\suml_{i=1}^{d} f(s_i)\right]$, where $f(s) = \frac{\|\wo\|^2\sigma^2}{\|\wo\|^2ns+\sigma^2d}$.
The following lemma (proof in \Apxref{asubsec:other_proofs}) is a simple application of Jensen's inequality and aids us in completing the proof
\begin{lemma}\label{lem:convexity_S}
For a function convex function $f(\cdot):\R\rightarrow\R$, we have 
\begin{align*}
	\ex_{\mS}\left[\suml_{i=1}^d f(s_i)\right]\ge \begin{cases}
		df(1) & \text{if $n\ge d$}\\
		nf(\frac{d}{n}) + (d-n)f(0) & \text{if $n<d$}
	\end{cases}
\end{align*}
where the expectation is over $\mS$ is for the distribution of eigenvalues of $\Sigma_\mX$ when $\mX\sim\gN(0,I_d)^n$.
\end{lemma}
By noticing that $f(\cdot,\lambda)$ is convex in the first argument, using \Lemref{lem:convexity_S} and the fact that $f(0)=\frac{\|\wo\|^2}{d}$ and $f(1)=\frac{\|\wo\|^2\sigma^2}{\|\wo\|^2n+\sigma^2d}$ and $f(\frac{d}{n})=\frac{\|\wo\|^2\sigma^2}{\|\wo\|^2d+\sigma^2d}$, we get 
\begin{align*}
	\gE_n(\gdreg(\cdot;\vw_0),\mu_\wo) \ge \begin{cases}
		df(1)=\frac{d\|\wo\|^2\sigma^2}{\|\wo\|^2n+\sigma^2d} & \text{if $n\ge d$}\\
		nf(\frac{d}{n}) + (d-n)f(0)= n\frac{\|\wo\|^2\sigma^2}{\|\wo\|^2d+\sigma^2d} + (d-n)\frac{\|\wo\|^2}{d}& \text{if $n<d$}
	\end{cases}
\end{align*}
which completes the proof for the first part of the theorem.

For the second part where $\|\wo\|=\sigma=r\ge 1$ and $\epsilon\in\left(0,\frac{r^2}{2}\right)$, it is not difficult to see that $\gE_n(\gdreg(\cdot,\vw_0),\mu_\wo)\ge \frac{dr^2}{n+d},~\forall n>0$.
To find the minimum $n\ge d$ such that $\gE_n(\gdreg(\cdot,\vw_0),\mu_\wo)\le\epsilon$, we observe the following
\begin{align*}
	\epsilon
	&\ge\gE_n(\gdreg(\cdot,\vw_0),\mu_\wo)
	\ge \frac{dr^2}{n+d}
	\implies n\ge d\left(\frac{r^2}{\epsilon}-1\right)
	\ge^{(a)} \frac{dr^2}{2\epsilon}
\end{align*}
where $(a)$ uses $\epsilon\le \frac{r^2}{2}$.
This gives us $n_\epsilon(\gdreg(\cdot,\vw_0))\ge \frac{dr^2}{2\epsilon}$ as desired.
\end{proof}

We now prove the result for $\gdstep$.
\begin{lemmano}[{\bf \Thmref{thm:lower_bound_gd}(b)}]
For every $\vw_0\in\R^d$, number of samples needed to have $\epsilon$ excess risk on a new task is
\begin{align*}
	\gE_n(\gdstep(\cdot;\vw_0),\mu_\wo) \ge \begin{cases}
		\frac{d\|\wo\|^2\sigma^2}{\|\wo\|^2n+\sigma^2d} & \text{if $n\ge d$}\\\\
		\frac{n}{d}\frac{\|\wo\|^2\sigma^2}{\|\wo\|^2+\sigma^2} + \frac{(d-n)}{d}\|\wo\|^2& \text{if $n<d$}
	\end{cases}
\end{align*}
Furthermore if $\|\wo\|=\sigma=r\ge 1$ and $\epsilon\in\left(0,\frac{r^2}{2}\right)$ , then the number of samples needed to achieve excess error of $\epsilon$ is
\begin{align*}
	\min\limits_{\eta\ge 0,t_0\in\mathbb{N}}~~n_\epsilon(\gdstep(\cdot;\vw_0), \mu_\wo) \ge \frac{dr^2}{2\epsilon}
\end{align*}
\end{lemmano}
\begin{proof}[Proof of \Thmref{thm:lower_bound_gd}(b)]
From \Lemref{lem:gd_closed_form} we have
$$\gdstep(S; \vw_{0}) = (I_d-\eta\Sigma_\mX)^{t_0}\vw_{0} + (I_d-(I_d-\eta\Sigma_\mX)^{t_0})\Sigma_\mX^{\dagger}\Sigma_\mX\vv + \frac{1}{n}(I_d-(I_d-\eta\Sigma_\mX)^{t_0})\Sigma_\mX^{\dagger}\mX^{\top}\xi$$
Instantiating \Lemref{lem:excess_risk_linear} with $\mA_\mX=(I_d-\eta\Sigma_\mX)^{t_0}$, $\mB_\mX=(I_d-(I_d-\eta\Sigma_\mX)^{t_0})\Sigma_\mX^{\dagger}\Sigma_\mX$ and $\mC_\mX=\frac{1}{n}(I_d-(I_d-\eta\Sigma_\mX)^{t_0})\Sigma_\mX^{\dagger}\mX^{\top}$, we get from a similar calculation to $\gdreg$
\begin{align*}
	\gE_n&(\gdstep(\cdot; \vw_{0}),\mu_\wo)\\
	&=\ex_{\mX}\left[\|(I_d-(I_d-(I_d-\eta\Sigma_\mX)^{t_0})\Sigma_\mX^{\dagger}\Sigma_\mX)\wo\|^2\right]+\ex_{\mX}\frac{1}{n^2}\left[\sigma^2\tr\left(\mX\Sigma_\mX^{\dagger}(I_d-(I_d-\eta\Sigma_\mX)^{t_0})^2\Sigma_\mX^{\dagger}\mX^{\top}\right)\right]\\
	&=\ex_{\mX}\left[\|(I_d-\Sigma_\mX^{\dagger}\Sigma_\mX+(I_d-\eta\Sigma_\mX)^{t_0}\Sigma_\mX^{\dagger}\Sigma_\mX)\wo\|^2\right]+\frac{\sigma^2}{n^2}\ex_{\mX}\left[\tr\left((I_d-(I_d-\eta\Sigma_\mX)^{t_0})^2\Sigma_\mX^{\dagger}\mX^{\top}\mX\Sigma_\mX^{\dagger}\right)\right]\\
	&=\ex_{\mX}\left[\|(I_d-\Sigma_\mX^{\dagger}\Sigma_\mX+(I_d-\eta\Sigma_\mX)^{t_0}\Sigma_\mX^{\dagger}\Sigma_\mX)\wo\|^2\right]+\frac{\sigma^2}{n}\ex_{\mX}\left[\tr\left((I_d-(I_d-\eta\Sigma_\mX)^{t_0})^2\Sigma_\mX^{\dagger}\Sigma_\mX\Sigma_\mX^{\dagger}\right)\right]\\
	&=\underbrace{\ex_{\mX}\left[\|(I_d-\Sigma_\mX^{\dagger}\Sigma_\mX+(I_d-\eta\Sigma_\mX)^{t_0}\Sigma_\mX^{\dagger}\Sigma_\mX)\wo\|^2\right]}_{bias(\gdstep)}+\underbrace{\frac{\sigma^2}{n}\ex_{\mX}\left[\tr\left((I_d-(I_d-\eta\Sigma_\mX)^{t_0})^2\Sigma_\mX^{\dagger}\right)\right]}_{var(\gdstep)}
\end{align*}
We separately analyze the bias and variance terms
\paragraph{Bias:} The bias term can be handled similarly by noticing that
\begin{align*}
	bias(\gdstep)
	&= \ex_{\mX}\left[\wo^{\top}(I_d-\Sigma_\mX^{\dagger}\Sigma_\mX+(I_d-\eta\Sigma_\mX)^{t_0}\Sigma_\mX^{\dagger}\Sigma_\mX)^2\wo\right]\\
	&= \ex_{\mV,\mS}\left[\wo^{\top}(\mV\mV^{\top}-\mV\mS^{\dagger}\mV^{\top}\mV\mS\mV^{\top}+(\mV\mV^{\top}-\eta\mV\mS\mV^{\top})^{t_0}\mV\mS^{\dagger}\mV^{\top}\mV\mS\mV^{\top})^2\wo\right]\\
	&= \ex_{\mS}\ex_{\mV}\left[\wo^{\top}\mV(I_d-\mS^{\dagger}\mS+(I_d-\eta\mS)^{t_0}\mS^{\dagger}\mS)^2\mV^{\top}\wo\right]\\
	&= \ex_{\mS}\ex_{\mV}\left[\suml_{i=1}^{d} h(s_i,\eta,t_0)(\wo^{\top}\vv_i)^2\right] = \ex_{\mS}\left[\suml_{i=1}^{d} h(s_i,\eta,t_0)\ex_{\vv_i}(\wo^{\top}\vv_i)^2\right]\\
	&= \frac{\|\wo\|^2}{d}\ex_{\mS}\left[\suml_{i=1}^{d} h(s_i,\eta,t_0)\right],\text{ where}\\
	h(s,\eta,t_0) &= (1-\eta s)^{2t_0}
\end{align*}

\paragraph{Variance:} We now look at the variance term
\begin{align*}
	var(\gdstep)
	&= \frac{\sigma^2}{n}\ex_{\mX}\left[\tr((I_d-(I_d-\eta\Sigma_\mX)^{t_0})^2\Sigma_\mX^{\dagger})\right]
	= \frac{\sigma^2}{n}\ex_{\mX}\left[\tr((\mV\mV^{\top}-(\mV\mV^{\top}-\eta\mV\mS\mV^{\top})^{t_0})^2\mV\mS^{\dagger}\mV^{\top})\right]\\
	&= \frac{\sigma^2}{n}\ex_{\mX}\left[\tr(\mV(I_d-(I_d-\eta\mS)^{t_0})^2\mV^{\top}\mV\mS^{\dagger}\mV^{\top})\right]
	= \frac{\sigma^2}{n}\ex_{\mX}\left[\tr(\mV(I_d-(I_d-\eta\mS)^{t_0})^2\mS^{\dagger}\mV^{\top})\right]\\
	&= \frac{\sigma^2}{n}\ex_{\mS}\left[\tr((I_d-(I_d-\eta\mS)^{t_0})^2\mS^{\dagger})\right]
	= \frac{\sigma^2}{n}\ex_{\mS}\left[\suml_{i=1}^dg(s_i,\eta,t_0)\right]\text{, where}\\
	g(s,\eta,t_0) &= \begin{cases}
		\frac{(1-(1-\eta s)^{t_0})^2}{s} & \text{if $s>0$}\\
		0 & \text{if $s=0$}
	\end{cases}
\end{align*}
Thus we get the following lower bound for the excess risk
\begin{align*}
	\gE_n(\gdstep(\cdot;\vw_0),\mu_\wo)
	&= bias(\gdstep) + var(\gdstep)
	\ge \frac{\|\wo\|^2}{d}\ex_{\mS}\left[\suml_{i=1}^{d} h(s_i,\eta,t_0)\right] + \frac{\sigma^2}{n}\ex_{\mS}\left[\suml_{i=1}^dg(s_i,\eta,t_0)\right]\\
	&= \ex_{\mS}\left[\suml_{i=1}^{d}\frac{\|\wo\|^2}{d} h(s_i,\eta,t_0) + \frac{\sigma^2}{n}g(s_i,\eta,t_0)\right]
\end{align*}
We will again show that $\frac{\|\wo\|^2}{d} h(s,\eta,t_0) + \frac{\sigma^2}{n}g(s,\eta,t_0) \ge \frac{\|\wo\|^2\sigma^2}{\|\wo\|^2ns+\sigma^2d}$.
Again, this is obvious for $s=0$ from the definitions of $h$ and $g$.
For $s>0$, we can write $\frac{\|\wo\|^2}{d} h(s,\eta,t_0) + \frac{\sigma^2}{n}g(s,\eta,t_0)=\frac{\|\wo\|^2}{d}(1-\eta s)^{2t_0}+\frac{\sigma^2}{sn}(1-(1-\eta s)^{t_0})^2$, which is again of the form $\alpha a^2 + \beta b^2$ with $a=(1-\eta s)^{t_0},b=(1-(1-\eta s)^{t_0})$ satisfying $a+b=1$.
Thus \Lemref{lemma:asq_bsq} gives us the desired inequality, which directly implies $\gE_n(\gdreg(\cdot;\vw_0),\mu_\wo)\ge \ex_{\mS}\left[\suml_{i=1}^{d} f(s_i)\right]$, where $f(s) = \frac{\|\wo\|^2\sigma^2}{\|\wo\|^2ns+\sigma^2d}$.
This is exactly the same lower bound as in the proof of $\gdreg$, and thus the theorem follows from identical arguments.
\end{proof}

\subsection{Closed form solutions}\label{asubsec:closed_form}
We now prove \Lemref{lem:gd_closed_form}.
Before that, we will state and prove the following simple lemmas about linear dynamics that will be useful later.
\begin{lemma}\label{lem:linear_dynamics}
For a symmetric psd matrix $\mM\in\R^{d\times d}$, let $\mM=\mB\mS\mB^{-1}$ be its diagonalization. For $\vb\in\R^d$ that is in the range of $\mM$, the solution to the system $\frac{d\vw_t}{dt} = -\mM\vw_t+\vb$ starting from $\vw_0$ is
\begin{align*}
	\vw_t &= \mB e^{-t\mS}\mB^{-1}\vw_0 + \mB(I_d-e^{-t\mS})\mS^{\dagger}\mB^{-1}\vb\\
	\vw_{\infty} &= (I_d-\mM^{\dagger}\mM)\vw_0 + \mM^{\dagger}\vb
\end{align*}
where for a diagonal matrix $\mS=diag(s_1,\dots,s_d)$, $e^{-t\mS}$ is defined as $diag(e^{-ts_1},\dots,e^{-ts_d})$ and $\mS^{\dagger}$ is a diagonal matrix with $\mS^{\dagger}(i,i)=s_i^{-1}$ if $s_i>0$, otherwise $\mS^{\dagger}(i,i)=0$.
\begin{proof}
Since $\vb$ is the range of $\mM$, let $\vb=\mM\bar{\vb}$.
The dynamics $\frac{d\vw_t}{dt} = -\mB\mS\mB^{-1}\vw_t+\mB\mS\mB^{-1}\bar{\vb}$ can be rewritten as $\frac{d(\mB^{-1}\vw_t)}{dt}=-\mS(\mB^{-1}\vw_t+\mB^{-1}\bar{\vb})$.
Setting $\tilde{\vw}_t=\mB^{-1}\vw_t$ and $\tilde{\vb}=\mB^{-1}\bar{\vb}$, we get $\frac{d\tilde{\vw}_t}{dt}=-\mS(\tilde{\vw}_t+\tilde{\vb})$.
Since $\mS=diag(s_1,\dots,s_d)$ is a diagonal matrix, we can decouple the dynamics
\begin{align*}
	\frac{d\tilde{\vw}_t(i)}{dt}=-s_i(\tilde{\vw}_t(i)-\tilde{\vb}(i)), ~\forall i\in[d]
\end{align*}
These scalar dynamics can be solved and it can be verified easily that $\tilde{\vw}_t(i)=e^{-ts_i}\tilde{\vw}_0(i)+(1-e^{-ts_i})\tilde{\vb}(i)$.
By observing that $\bar{\vb}=\mM^{\dagger}\vb$ and $\tilde{\vb}=\mB^{-1}\bar{\vb}=\mB^{-1}\mM^{\dagger}\vb=\mS^{\dagger}\mB^{-1}\vb$, we can summarize the dynamics as $\tilde{\vw}_t=e^{-t\mS}\tilde{\vw}_0 + (I_d-e^{-t\mS})\mS^{\dagger}\mB^{-1}\vb$.
Using $\tilde{\vw}_0=\mB^{-1}\vw_0$ and $\vw_t=\mB\tilde{\vw}_t$, multiplying by $\mB$ on both sides completes the first part of the proof, i.e. $\vw_t = \mB e^{-t\mS}\mB^{-1}\vw_0 + \mB(I_d-e^{-t\mS})\mS^{\dagger}\mB^{-1}\vb$.
Furthermore, as $t\rightarrow\infty$, we see that $e^{-t\mS}\rightarrow diag(\1\{s_1=0\},\dots,\1\{s_d=0\})$ since for $s_i\neq0$, $e^{-ts_i}\rightarrow 0$ while if $s_i=0$ then $e^{-ts_i}=1$ for every $t\in\R$.
This completes the second part of the proof.
\end{proof}
\end{lemma}
\begin{lemma}\label{lem:linear_dynamics_step}
For a symmetric psd matrix $\mM\in\R^{d\times d}$, let $\mM=\mB\mS\mB^{-1}$ be its diagonalization. For $\vb\in\R^d$ that is in the range of $\mM$, the solution to the system $\vw_{t+1}-\vw_t = -\eta(\mM\vw_t-\vb)$ starting from $\vw_0$ is
\begin{align*}
	\vw_t &= \mB (I_d-\eta\mS)^t\mB^{-1}\vw_0 + \mB(I_d- (I_d-\eta\mS)^t)\mS^{\dagger}\mB^{-1}\vb\\
	&= (I_d-\eta\mM)^t\vw_0 + (I_d- (I_d-\eta\mM)^t)\mM^{\dagger}\vb
\end{align*}
\begin{proof}
Since $\vb$ is the range of $\mM$, let $\vb=\mM\bar{\vb}$.
The dynamics $\vw_{t+1}-\vw_t = -\eta\mB\mS\mB^{-1}\vw_t+\eta\mB\mS\mB^{-1}\bar{\vb}$ can be rewritten as $\mB^{-1}\vw_{t+1}=-(I_d-\eta\mS)\mB^{-1}\vw_t+\eta\mS\mB^{-1}\bar{\vb}$.
Setting $\tilde{\vw}_t=\mB^{-1}\vw_t$ and $\tilde{\vb}=\mB^{-1}\bar{\vb}$, we get $\tilde{\vw}_{t+1}=-(I_d-\eta\mS)\tilde{\vw}_t+\eta\mS\tilde{\vb}$.
Since $\mS=diag(s_1,\dots,s_d)$ is a diagonal matrix, we can decouple the dynamics, for every $i\in[d]$,
\begin{align*}
	\tilde{\vw}_{t+1}(i) &= -(1-\eta s_i)\tilde{\vw}_{t}(i)+\eta s_i\tilde{\vb}(i)\\
	&= -(1-\eta s_i)^{t+1}\tilde{\vw}_{0}(i) + \eta s_i \left(\suml_{j=0}^t (1-\eta s_i)^j\right) \tilde{\vb}(i)
\end{align*}
This can be simplified to eventually get $\tilde{\vw}_t(i)=(1-\eta s_i)^t\tilde{\vw}_0(i)+(1-(1-\eta s_i)^t)\tilde{\vb}(i)$.
By observing that $\bar{\vb}=\mM^{\dagger}\vb$ and $\tilde{\vb}=\mB^{-1}\bar{\vb}=\mB^{-1}\mM^{\dagger}\vb=\mS^{\dagger}\mB^{-1}\vb$, we can summarize the dynamics as $\tilde{\vw}_t=(I_d-\eta\mS)^t\tilde{\vw}_0 + (I_d-(I_d-\eta\mS)^t)\mS^{\dagger}\mB^{-1}\vb$.
Using $\tilde{\vw}_0=\mB^{-1}\vw_0$ and $\vw_t=\mB\tilde{\vw}_t$, multiplying by $\mB$ on both sides completes the proof, i.e. $\vw_t = \mB (I_d-\eta\mS)^t\mB^{-1}\vw_0 + \mB(I_d- (I_d-\eta\mS)^t)\mS^{\dagger}\mB^{-1}\vb$.
By observing that $\mB(I_d-\eta\mS)^t\mB^{-1}=(I_d-\eta\mB\mS\mB^{-1})^t=(I_d-\eta\mM)^t$, we get
\begin{align*}
	\vw_t &= \mB(I_d-\eta\mS)^t\mB^{-1}\vw_0 + \mB(I_d- (I_d-\eta\mS)^t)\mB^{-1}\mB\mS^{\dagger}\mB^{-1}\vb\\
	&= (I_d-\eta\mM)^t\vw_0 + (\mB\mB^{-1}- \mB(I_d-\eta\mS)^t\mB^{-1})\mB\mS^{\dagger}\mB^{-1}\vb\\
	&= (I_d-\eta\mM)^t\vw_0 + (I_d- (I_d-\eta\mM)^t)\mM^{\dagger}\vb
\end{align*}
which completes the proof
\end{proof}
\end{lemma}

\paragraph{Proving \Lemref{lem:gd_closed_form}:} 
We restate the statement of the \Lemref{lem:gd_closed_form} here for convenience.
\begin{lemmano}[{\bf \Lemref{lem:gd_closed_form}}]
Let $S=(\mX,\xi)$ be a sample from $\rho_{\vv}$, $\mX\in\R^{n\times d}, \xi\in\R^n$. Let $\Sigma_\mX=\frac{1}{n}\suml_{i=1}^{n}\vx_i\vx_i^{\top}\in\R^{d\times d}$
\begin{align*}
	\gdreg(S; \vw_{0}) &= (I_d-(\Sigma_\mX+\lambda I_d)^{\dagger}(\Sigma_\mX+\lambda I_d))\vw_{0} + (\Sigma_\mX+\lambda I_d)^{\dagger}\Sigma_\mX\vv + \frac{1}{n}(\Sigma_\mX+\lambda I_d)^{\dagger}\mX^{\top}\xi
\end{align*}
\begin{align*}
	\gdstep(S; \vw_{0}) &= (I_d-\eta\Sigma_\mX)^{t_0}\vw_{0} + (I_d-(I_d-\eta\Sigma_\mX)^{t_0})\Sigma_\mX^{\dagger}\Sigma_\mX\vv + \frac{1}{n}(I_d-(I_d-\eta\Sigma_\mX)^{t_0})\Sigma_\mX^{\dagger}\mX^{\top}\xi
\end{align*}
\end{lemmano}
\begin{proof}[Proof of \Lemref{lem:gd_closed_form}]
We first prove the result for $\gdreg$.
Recall the definition of the regularized loss from \Eqref{eq:reg_loss} and the dynamics for $\gdreg$
\begin{align*}
	\ell_{S,\lambda}(\vw)=\frac{1}{n}\suml_{i=1}^{n}(\vw^{\top}\vx_i-y_i)^2 + \frac{\lambda}{2}\|\vw\|^2;~~~\frac{d\vw_t}{dt} = -\nabla_{\vw}\ell_{S,\lambda}(\vw_t)
\end{align*}
where $y_i = \vv^{\top}\vx_i + \xi_i$. The gradient of $\ell_{S,\lambda}$ is
\begin{align*}
	\nabla_{\vw}\ell_{S,\lambda}(\vw)
	&= \frac{1}{n}\suml_{i=1}^{n}(\vw^{\top}\vx_i-y_i)\vx_i + \lambda \vw=\frac{1}{n}\suml_{i=1}^{n}(\vw^{\top}\vx_i-\vv^{\top}\vx_i-\xi_i)\vx_i + \lambda \vw\\
	&= \left(\frac{1}{n}\suml_{i=1}^{n}\vx_i\vx_i^{\top} + \lambda I_d\right)\vw - \left(\frac{1}{n}\suml_{i=1}^{n}\vx_i\vx_i^{\top}\right)\vv-\frac{1}{n}\mX^{\top}\xi\\
	&= (\Sigma_\mX + \lambda I_d)\vw - \Sigma_\mX\vv-\frac{1}{n}\mX^{\top}\xi
\end{align*}
If $\mM = (\Sigma_\mX + \lambda I_d)$ and $\vb = \Sigma_\mX\vv+\frac{1}{n}\mX^{\top}\xi$, then $\nabla_{\vw}\ell_{S,\lambda}(\vw)=\mM\vw - \vb$ and the dynamics are $\frac{d\vw_t}{dt} = -\mM\vw_t+\vb$.
Note that $\vb$ is in the range of $\mM$ for every $\lambda\ge0$; this is obvious for $\lambda>0$ when $\mM$ is full rank, but even $\lambda=0$, since $\vb$ lies in the span of rows of $\mX$, it lies in the span of $\Sigma_\mX=\frac{1}{n}\mX^{\top}\mX$.
Thus by \Lemref{lem:linear_dynamics}, we get that $\vw_{\infty}=(I_d-\mM^{\dagger}\mM)\vw_0 + \mM^{\dagger}\vb$.
Plugging in values of $\mM$ and $\vb$ gives the desired closed form for $\gdreg$.

We now derive the closed form solution for $\gdstep$.
Recall the dynamics of $\gdstep$
\begin{align*}
	\ell_S(\vw)=\frac{1}{n}\suml_{i=1}^n(\vw^{\top}\vx_i-y_i)^2; ~~\vw_{t+1} = \vw_{t} - \eta \nabla_{\vw}\ell_S(\vw_t)
\end{align*}
where again $y_i=\vv^{\top}\vx_i+\xi_i$. The gradient of $\ell_S$ is
\begin{align*}
	\nabla_{\vw}\ell_{S}(\vw)
	&= \frac{1}{n}\suml_{i=1}^{n}(\vw^{\top}\vx_i-y_i)\vx_i=\frac{1}{n}\suml_{i=1}^{n}(\vw^{\top}\vx_i-\vv^{\top}\vx_i-\xi_i)\vx_i\\
	&= \left(\frac{1}{n}\suml_{i=1}^{n}\vx_i\vx_i^{\top}\right)\vw - \left(\frac{1}{n}\suml_{i=1}^{n}\vx_i\vx_i^{\top}\right)\vv-\frac{1}{n}\mX^{\top}\xi\\
	&= \Sigma_\mX\vw - \Sigma_\mX\vv-\frac{1}{n}\mX^{\top}\xi
\end{align*}
Setting $\mM=\Sigma_\mX$ and $\vb=\Sigma_\mX\vv + \frac{1}{n}\mX^{\top}\xi$, we get $\nabla_\vw\ell_S(\vw)=\mM\vw-\vb$ and the dynamics are $\vw_{t+1}-\vw_t=-\eta(\mM\vw_t-\vb)$.
Again since $\vb$ is in the span of $\mM$, we can use \Lemref{lem:gd_closed_form} to get $\vw_{t_0}=(I_d-\eta\mM)^{t_0}\vw_0 + (I_d- (I_d-\eta\mM)^{t_0})\mM^{\dagger}\vb$.
Plugging in the values of $\mM$ and $\vb$ completes the proof for $\gdstep$.
\end{proof}

\subsection{Other proofs}\label{asubsec:other_proofs}
\paragraph{Proving \Lemref{lem:excess_risk_linear}}
\begin{proof}[Proof of \Lemref{lem:excess_risk_linear}]
We start by looking at the loss for $s\wo$ for $s\in\{\pm1\}$
\begin{align*}
	\ex_{S\sim\rho_{s\wo}}\left[\ell_{s\wo}(\Alg(\cdot;\vw_0))-\sigma^2\right]
	&=\ex_{S\sim\rho_{s\wo}}\left[\|\Alg(S;\vw_0))-s\wo\|^2\right]\\
	&=\ex_{\substack{\mX\sim\gN^n(0,I_d)\\\xi\sim\gN(0,\sigma^2I_m)}}\left[\|\mA_\mX\vw_{0} + \mB_\mX (s\wo) + \mC_\mX\xi-s\wo\|^2\right]\\
	&=\ex_{\mX,\xi}\|\mA_\mX\vw_{0} - (I_d-\mB_\mX)s\wo + \mC_\mX\xi\|^2\\
	&=^{(a)}\ex_{\mX}\|(\mA_\mX\vw_{0} - (I_d-\mB_\mX)s\wo\|^2 + \ex_{\mX,\xi}\|\mC_\mX\xi\|^2\\
	&=^{(b)}\ex_{\mX}\|(s\mA_\mX\vw_{0} - (I_d-\mB_\mX)\wo\|^2 + \ex_{\mX,\xi}\|\mC_\mX\xi\|^2
\end{align*}
where $(a)$ uses the fact that $\mX$ and $\xi$ are independent and $\ex_{\xi}\xi=0$ and $(b)$ uses $s^2=1$. Thus we get,
\begin{align*}
	\gE_n(\gdreg(\cdot; \vw_{0}),\mu_\wo)
	&=\ex_{s\sim\{\pm1\}}\left[\ex_{S\sim\rho_{s\wo}}[\ell_{s\wo}(\gdreg(\cdot;\vw_0))\right]-\sigma^2\\
	&= \ex_{s\sim\{\pm1\}}\left[\ex_{S\sim\rho_{s\wo}}[\ell_{s\wo}(\gdreg(\cdot;\vw_0))-\sigma^2]\right]\\
	&= \ex_{s\sim\{\pm1\}}\ex_{\mX,\xi}\left[\|(s\mA_\mX\vw_{0} - (I_d-\mB_\mX)\wo\|^2 + \|\mC_\mX\xi\|^2\right]\\
	&= \ex_{\mX}\ex_{s\sim\{\pm1\}}\left[\|(s\mA_\mX\vw_{0} - (I_d-\mB_\mX)\wo\|^2\right] + \ex_{\mX,\xi}\left[\xi^{\top}\mC_\mX^{\top}\mC_\mX\xi\right]\\
	&\ge^{(a)} \ex_{\mX}\left[\|(\ex_{s\sim\{\pm1\}}[s]\mA_\mX\vw_{0} - (I_d-\mB_\mX)\wo\|^2\right] + \ex_{\mX,\xi}\left[\xi^{\top}\mC_\mX^{\top}\mC_\mX\xi\right]\\
	&=^{(b)} \ex_{\mX}\left[\|(I_d-\mB_\mX)\wo\|^2\right] + \ex_{\mX}\left[\ex_{\xi}\tr(\mC_\mX^{\top}\mC_\mX\xi\xi^{\top})\right]\\
	&=^{(c)} \ex_{\mX}\left[\|(I_d-\mB_\mX)\wo\|^2\right] + \ex_{\mX}[\sigma^2\tr(\mC_\mX^{\top}\mC_\mX)]
\end{align*}
where $(a)$ is true by convexity of the quadratic function in $s$, $(b)$ uses $\xi^{\top}\mP\xi=\tr(\mP\xi\xi^{\top})$ for any $d\times d$ matrix $\mP$, $(c)$ uses the linearity of $\tr$ operator and the fact that $\ex_{\xi}\xi\xi^{\top}=\sigma^2I_m$ when $\xi\sim\gN(0,\sigma^2I_m)$.
This completes the proof.
Note that we only needed first and second moment conditions on $\xi$ to prove this lemma.
\end{proof}

\paragraph{Proving \Lemref{lem:convexity_S}}
\begin{proof}[Proof of \Lemref{lem:convexity_S}]
If $n\ge d$, we just follow the steps below that heavily use Jensen's inequality due to the convexity of $f$.
\begin{align*}
	\ex_{\mS}\left[\suml_{i=1}^{d} f(s_i)\right]
	&= \ex_{\mS}\left[d\ex_{i\sim[d]} [f(s_i)]\right]
	\ge^{(a)} d\ex_{\mS}\left[f\left(\ex_{i\sim[d]} [s_i]\right)\right]
	= d\ex_{\mS}\left[f\left(\frac{1}{d}\tr(\mS)\right)\right]
	=^{(b)} d\ex_{\mX}\left[f\left(\frac{1}{d}\tr(\Sigma_\mX)\right)\right]\\
	&\ge^{(c)} df\left(\frac{1}{d}\ex_{\mX}[\tr(\Sigma_\mX)]\right)
	= df\left(\frac{1}{d}\tr(\ex_{\mX}\Sigma_\mX)\right)
	=^{(d)} df\left(\frac{1}{d}\tr(I_d)\right)
	= df\left(1\right)
\end{align*}
where $(a)$ follows from Jensen's inequality, $(b)$ follows from the fact that $\tr(\mS)=\tr(\mS\mV^{\top}\mV)=\tr(\mV\mS\mV^{\top})=\tr(\Sigma_\mX)$, $(c)$ follows from Jensen's inequality and $(d)$ follows from $\ex_{\mX}\Sigma_\mX=\ex_{\mX}\frac{1}{n}\suml_{i=1}^{n}\vx_i\vx_i^{\top}=\frac{1}{n}\suml_{i=1}^{n}\ex_{\mX}\vx_i\vx_i^{\top}=\frac{1}{n}\suml_{i=1}^{n}I_d=I_d$.

When $n<d$, we know that $\mX$ (and hence $\Sigma_\mX$) has rank at most $n<d$, thus the $d-n$ smallest eigenvalues are 0, i.e. $s_i=0$ for $n+1\le i\le d$.
Note that $\suml_{i=1}^d s_i = \suml_{i=1}^{n}s_i = \tr(\mS)$.
Following the steps below,
\begin{align*}
	\ex_{\mS}\left[\suml_{i=1}^{d} f(s_i)\right]
	&=\ex_{\mS}\left[\suml_{i=1}^{n} f(s_i)\right] + \ex_{\mS}\left[\suml_{i=n+1}^{d} f(s_i)\right]
	= \ex_{\mS}\left[n\ex_{i\sim[n]} [f(s_i)]\right] + \ex_{\mS}\left[\suml_{i=n+1}^{d} f(0)\right]\\
	&\ge \ex_{\mS}\left[n f(\ex_{i\sim[n]}[s_i])\right] + (d-n)f(0)
	= \ex_{\mS}\left[n f\left(\frac{1}{n}\tr(\mS)\right)\right] + (d-n)f(0)\\
	&= \ex_{\mX}\left[n f\left(\frac{1}{n}\tr(\Sigma_\mX)\right)\right] + (d-n)f(0)
	\ge n f\left(\frac{1}{n}\tr(\ex_{\mX}\Sigma_\mX)\right) + (d-n)f(0)\\
	&= nf\left(\frac{d}{n}\right) + (d-n)f(0)
\end{align*}
\end{proof}

\section{Non-convex proofs}\label{asec:non_convex}

\subsection{Theorems and Lemmas for Reptile}\label{asubsec:reptile}
Let $\bar{\wf}_{i+1},\bar{\ws}_{i+1}=\gdpop(\ell_{\rho_{i+1}},(\wf_{i},\ws_{i}))$ be the solution for task $\rho_{i+1}$ that is found by gradient descent starting from current initialization.
Thus the reptile update is $\wf_{i+1}=(1-\tau)\wf_{i}+\tau\bar{\wf}_{i+1}$ and $\ws_{i+1}=(1-\tau)\ws_{i}+\tau\bar{\ws}_{i+1}$.
Let $\bwo=\wo/\|\wo\|$ be the unit vector and let $r=\|\wo\|$.
\begin{lemma}\label{lem:reptile_dynamics}
Given a sequence of tasks $\rho_{1:T}$ where $\rho_i=\rho_{s_i\wo}$ for $s_i\in\{\pm1\}$. Starting with $\wf_0=\kappa I_d, \ws=\vzero_d$, then the initialization learned by \reptile ~satisfies the following at every step $i\in[T]$
\begin{align*}
\wf_{i} = (a_i-\kappa)\bwo{\bwo}+\kappa I_d,~\ws_{i} = b_i\bwo
\end{align*}
where
\begin{align*}
	a_0&=\kappa,~b_0=0\\
	c_i&=a_i^2-b_i^2, ~a_{i+1}=(1-\tau)a_i + \tau\bar{a}_{i+1}, ~b_{i+1}=(1-\tau)b_i + \tau~s_{i+1}\bar{b}_{i+1}\\
	\bar{a}_{i+1}&=\sqrt{\frac{c_i+\sqrt{4r^2+c_i^2}}{2}}, \bar{b}_{i+1}=\sqrt{\frac{-c_i+\sqrt{4r^2+c_i^2}}{2}}
\end{align*}
\end{lemma}

The following key lemma about the solution of gradient flow for a single task starting from an initialization is crucial to prove the above lemma.
\begin{lemma}\label{lem:gd_dynamics}
Starting from $\wf(0)=(a(0)-\kappa)\bwo\bwo^{\top} + \kappa I_d$, $\ws(0)=b(0)\bwo$, with $a(0)>b(0)$, the solution of gradient flow on loss $\ell_{s\wo}$ for $s\in\{\pm1\}$, is $\bar{\wf},\bar{\ws}=\gdpop(\ell_{s\wo}, (\wf,\ws))$, where
\begin{align*}
	\bar{\wf}=(\bar{a}-\kappa)\bwo\bwo^{\top} &+ \kappa I_d, ~\bar{\ws}=\bar{b}\bwo~\text{, where}\\
	\bar{a} = \sqrt{\frac{c+\sqrt{4r^2+c^2}}{2}}, ~\bar{b} &= s\sqrt{\frac{-c+\sqrt{4r^2+c^2}}{2}}, ~c=a(0)^2-b(0)^2
\end{align*}
\end{lemma}

\begin{proof}[Proof of \Lemref{lem:reptile_dynamics}]
We prove this using a simple induction by assuming \Lemref{lem:gd_dynamics}.
It is clear for $i=0$ that $a_0=\kappa$ and $b_0=0$.
Suppose $\wf_i=(a_i-\kappa)\bwo{\bwo}+\kappa I_d,~\ws_{i} = b_i\bwo$.
From \Lemref{lem:gd_dynamics}, we get that $\bar{\wf}_{i+1}=(\bar{a}_i-\kappa)\bwo{\bwo}+\kappa I_d,~\ws_{i+1} = s_{i+1}\bar{b}_i\bwo$.
Doing the interpolation step completes the proof.
\end{proof}

\begin{proof}[Proof of \Lemref{lem:gd_dynamics}]
The proof uses ideas from \citet{saxe:14,saxe:19,gidel:19}, where the dynamics of linear networks is analyzed in the case where the subspace of the initialization is aligned with the target $s\bwo$.
We provide a proof of this lemma by borrowing the key ideas those works.
Let $\mU\in\R^{d\times d}$ be an orthonormal matrix, i.e. $\mU^{\top}\mU=I_d$, whose first column is $\bwo$.
Thus we can rewrite $\wf(0) = \mU\Lambda_1(0)\mU^{\top}$, where $\Lambda_1(0)\in\R^{d\times d}$ is a diagonal matrix that looks like $\Lambda_1(0) = \diag(a(0),\kappa,\dots,\kappa)$, $\ws(0)=\mU\Lambda_2(0)$, where $\Lambda_2(0)=(b(0),0,\dots,0)\in\R^d$ and $s\wo=\mU\Lambda_*$, where $\Lambda_*=(sr,0,\dots,0)\in\R^d$.
The loss to run gradient flow on is $\ell_{s\wo}(\wf,\ws)=\|\wf^{\top}\ws-s\wo\|^2$.
Dynamics of gradient flow is
\begin{align*}
	\frac{d\wf(t)}{dt}=-\nabla_{\wf}\ell_{s\wo}(\wf(t),\ws(t))=s\ws(t)\wo^{\top}-\ws(t)\ws(t)^{\top}\wf(t)\\
	\frac{d\ws(t)}{dt}=-\nabla_{\ws}\ell_{s\wo}(\wf(t),\ws(t))=s\wf(t)\wo-\wf(t)\wf(t)^{\top}\ws(t)
\end{align*}
Just like \citet{saxe:14,gidel:19}, we define $\Lambda_1(t)=\mU^{\top}\wf(t)\mU$, $\Lambda_2(t)=\mU^{\top}\ws(t)$, $\Lambda_*=\mU^{\top}\wo$. Thus
\begin{align*}
	\frac{d\Lambda_1(t)}{dt}=\Lambda_2(t)\Lambda_*^{\top}-\Lambda_2(t)\Lambda_2(t)^{\top}\Lambda_1(t)\\
	\frac{d\Lambda_2(t)}{dt}=\Lambda_1(t)\Lambda_*-\Lambda_1(t)\Lambda_1(t)^{\top}\Lambda_2(t)
\end{align*}
By a similar argument, we see that the time derivative of $\Lambda_1$ is non-zero only for the first diagonal entry while the derivative of $\Lambda_2$ is non-zero only for the first entry.
Thus the entire dynamics can be summarized by the dynamics of two scalar values
\begin{align*}
	\frac{da(t)}{dt}&=b(t)sr-b(t)^2a(t)\\
	\frac{db(t)}{dt}&=a(t)sr-a(t)^2b(t)
\end{align*}
Using the hyperbolic change of coordinates of $(a(t),b(t))=(\sqrt{c}\cosh(\theta/2), \sqrt{c}\sinh(\theta/2))$ and the analysis in Appendix A from \citet{saxe:14}, we have that the fixed point of the dynamics is at $\bar{\theta}=\sinh^{-1}(2rs/c)$, thus giving the solutions
\begin{align*}
	\bar{a}&=\sqrt{c}\cosh(\bar{\theta}/2)=\sqrt{\frac{c+\sqrt{4r^2+c^2}}{2}}\\
	\bar{b}&=\sqrt{c}\sinh(\bar{\theta}/2)=\sqrt{\frac{-c+\sqrt{4r^2+c^2}}{2}}
\end{align*}
\end{proof}

We now prove the key theorem that shows how the reptile update amplifies the component of the first layer in the direction of $\wo$.
Precisely, it shows that with high probability over sampling of the training tasks, $a_T$ from \Lemref{lem:reptile_dynamics} is large for appropriate choice of $T$ and $\tau$.
\begin{theorem}\label{thm:ablowup}
Suppose $\{a_i,b_i,c_i\}_{i=1}^{T}$ follow the dynamics from \Lemref{lem:reptile_dynamics} with $\{s_1,\dots,s_T\}\sim\{\pm1\}^{T}$. Then with probability at least $1-\delta$, $a_T\ge\min\left\{\frac{\sqrt{r}}{2\sqrt{\tau\log(T/\delta)}}, \sqrt{r}\frac{(\tau T)^{1/4}}{2}\right\}$. Picking $\tau=T^{-1/3}\log(2T/\delta)^{-2/3}$, we get that $a_T\ge\frac{\sqrt{r}T^{1/6}\log(2T/\delta)^{-1/6}}{2}=\tilde{\Omega}(\sqrt{r}T^{1/6})$
\begin{proof}
The proof has 3 mains steps
\begin{itemize}
	\item {\bf Step 1:} $a_i$ is non-decreasing and the increment in $a_i$ is a decreasing function of $|a_ib_i|$. Also $|a_ib_i|\le r$.
	\item {\bf Step 2:} With high probability, $|b_i|$ is small
	\item {\bf Step 3:} Either $|a_ib_i|$ is small, which gives an increment in $a_i$, otherwise, or $a_i=\Omega(1/|b_i|)$ is large since $|b_i|$ is small
\end{itemize}
\paragraph{Step 1:}
We first prove that $a_i$ is non-decreasing, which happens if $\bar{a}_{i+1}\ge a_i$ for every $i$.
\begin{align*}
	\bar{a}_{i+1}^2-a_i^2&=\frac{c_i+\sqrt{4r^2+c_i^2}}{2}-a_i^2=\frac{a_i^2-b_i^2+\sqrt{4r^2+(a_i^2-b_i^2)^2}-2a_i^2}{2}\\
	&=\frac{\sqrt{4r^2+(a_i^2-b_i^2)^2}-(a_i^2+b_i^2)}{2}=\frac{\sqrt{4(r^2-a_i^2b_i^2)+(a_i^2+b_i^2)^2}-(a_i^2+b_i^2)}{2}
\end{align*}
Thus $|a_ib_i|<r$ will ensure that $a_i$ is non-decreasing.
We show that using induction, $|a_0b_0|=0$ and assume $|a_ib_i|\le r$.
Notice that since $\bar{a}_{i+1}^2-\bar{b}_{i+1}=a_i^2-b_i^2$, we have that $(|a_i|-|\bar{a}_{i+1}|)(|b_i|-|\bar{b}_{i+1}|)\ge 0$.
\begin{align*}
	|a_{i+1}b_{i+1}|&=|(1-\tau)^2a_ib_i+\tau^2\bar{a}_{i+1}\bar{b}_{i+1}+\tau(1-\tau)a_i\bar{b}_{i+1}+\tau(1-\tau)b_i\bar{a}_{i+1}\\
	&\le(1-\tau)^2|a_ib_i|+\tau^2|\bar{a}_{i+1}\bar{b}_{i+1}|+\tau(1-\tau)[|a_i||\bar{b}_{i+1}|+|b_i||\bar{a}_{i+1}|]\\
	&\le(1-\tau)^2r+\tau^2r+\tau(1-\tau)[|a_i||b_i|+|\bar{a}_{i+1}||\bar{b}_{i+1}|] + \tau(1-\tau)(|a_i|-|\bar{a}_{i+1}|)(|\bar{b}_{i+1}|-|b_i|)\\
	&\le(1-\tau)^2r+\tau^2r+\tau(1-\tau)[r+r]+0=r
\end{align*}
Thus finishing the first step in the proof.
\paragraph{Step 2:}
We now move to the second step about $|b_i|$ being small.
\begin{prop}
With probability at least $1-\delta$ over $\{s_1,\dots,s_T\}$, $|b_i|\le\sqrt{2r\tau\log(2T/\delta)}$, for every $i\in[T]$
\end{prop}
From the dynamics, we have $X_{i+1}=b_{i+1}-(1-\tau)b_i=\tau s_{i+1}\bar{b}_{i+1}$.
Note that $\bar{b}_{i+1}$ depends only on $s_{1:i}$ and $|\bar{b}_{i+1}|\le\sqrt{r}$, thus conditioned on $s_{1:i}$, $X_{i+1}$ is $\tau\sqrt{r}$ sub-gaussian and $\ex_{}[X_{i+1}|s_{1:i}]=\tau\ex_{}[s_{i+1}\bar{b}_{i+1}|s_{1:i}]=0=\tau\bar{b}_{i+1}\ex_{}[s_{i+1}|s_{1:i}]=0$. It is easy to verify that we can rewrite $b_{i+1}=X_{i+1}+(1-\tau)X_i+(1-\tau)^2X_{i-1}+\dots+(1-\tau)^iX_1$, where we also use the fact that $b_0=0$.
Using Markov's inequality we get
\begin{align*}
	\Pr(b_{i+1}>\nu)&=\Pr(e^{tb_{i+1}}>e^{t\nu})\le e^{-t\nu}\ex_{}e^{tb_{i+1}}=e^{-t\nu}\ex_{}e^{t\suml_{j=0}^{i+1}(1-\tau)^{j}X_{i+1-j}}\\
	&=e^{-t\nu}\ex_{}\prodl_{j=0}^{i+1}e^{t(1-\tau)^{j}X_{i+1-j}}= e^{-t\nu}\prodl_{j=0}^{i}\ex_{}[e^{t(1-\tau)^{j}X_{i+1-j}}|s_{1:i-j}]\\
	&\le^{(a)} e^{-t\nu}\prodl_{j=0}^{i}e^{\frac{t^2\tau^2r(1-\tau)^{2j}}{2}}=e^{-t\nu}e^{\frac{t^2\tau^2r\suml_{j=0}^{i}(1-\tau)^{2j}}{2}}\\
	&\le e^{-t\nu}e^{\frac{t^2\tau^2r\suml_{j=0}^{\infty}(1-\tau)^{2j}}{2}}=e^{-t\nu}e^{\frac{t^2\tau^2r}{2(1-(1-\tau)^2)}}=e^{-t\nu}e^{\frac{t^2\tau^2r}{2(2\tau-\tau^2)}}\\
	&\le^{(b)}e^{-t\nu}e^{\frac{t^2\tau r}{2}}
\end{align*}
Where for $(a)$ we use the fact that $(1-\tau)^{j}X_{i+1-j}$ is zero mean and $(1-\tau)^{2j}\tau^2r$-subgaussian when conditioned on $s_{1:i-j}$, and for (b) we use $\tau<1$.
Picking the optimal value of $t=\frac{\nu}{\tau r}$, we get $\Pr(b_{i+1}>\nu)\le e^{-\frac{\nu^2}{2\tau r}}$.
By using the symmetry of $b_{i+1}$ (since the sequence $\{-s_1,\dots,-s_T\}$ will give $-b_{i+1}$ instead), we get that $\Pr(b_{i+1}<-\nu)\le e^{-\frac{\nu^2}{2\tau r}}$ and by union bound we get that $\Pr(\forall i\in[T],~|b_{i}|>\nu)\le 2Te^{-\frac{\nu^2}{2\tau r}}$.
Setting $\nu=\sqrt{2r\tau\log(\frac{2T}{\delta})}$, we get $\Pr(\forall i\in[T],~|b_{i}|>\nu)\le\delta$
\paragraph{Step 3:}
Let $\gamma=\sqrt{2\tau\log(\frac{2T}{\delta})}$; from step 2 we have $|b_i|\le\sqrt{r}\gamma,\forall i\in[T]$.
An easy induction can also show that $|b_i|\le\sqrt{r}$.
To show $a_T$ is large, we assume that $a_T<\alpha$ for some $\alpha$ and see how large $T$ can be without leading to a contradiction.
We also assume that $\alpha\ge1$, this assumptions will be justified in the end.
Since $a_i$ is non-decreasing, we also get that $a_i\le\alpha\sqrt{r},\forall i\in[T]$.
If $a_ib_i\ge\frac{r}{\sqrt{2}}$ for any $i$, then we have $a_i\ge\frac{r}{\sqrt{2}b_i}\ge\frac{\sqrt{r}}{\sqrt{2}\gamma}$ which would finish the proof.
If $a_ib_i<\frac{r}{\sqrt{2}}$ for every $i$, then we will prove that there is at least a constant increment in $a_i$.
Let $\Delta_i = \bar{a}_{i+1}-a_i$; as shown in step 1, $\Delta_i\ge0$.
\begin{align*}
	(\Delta_i + a_i)^2 - a_i^2 &= \frac{c_i+\sqrt{4r^2+c_i^2}}{2} - a_i^2= \frac{a_i^2-b_i^2+\sqrt{4r^2+(a_i^2-b_i^2)^2}}{2} - a_i^2\\
	& = \frac{\sqrt{4(r^2-a_i^2b_i^2)+(a_i^2+b_i^2)^2}-(a_i^2+b_i^2)}{2}\\
	& \ge^{(a)} \frac{\sqrt{2r^2+(a_i^2+b_i^2)^2}-(a_i^2+b_i^2)}{2}\\
	& \ge^{(b)} \frac{\sqrt{2r^2+r^2(\alpha^2+1)^2}-r(\alpha^2+1)}{2}=r\frac{\sqrt{2+(\alpha^2+1)^2}-(\alpha^2+1)}{2}\\
	& = r\frac{1}{\sqrt{2+(\alpha^2+1)^2}+(\alpha^2+1)}\ge\frac{r}{\sqrt{2(\alpha^2+1)^2}+(\alpha^2+1)}\\
	& = \frac{r}{(\sqrt{2}+1)(\alpha^2+1)}
\end{align*}
where $(a)$ follows because $|a_ib_i|<\frac{r}{\sqrt{2}}$ and $(b)$ follows from the fact that $\sqrt{x+y^2}-y\ge\sqrt{x+z^2}-z$ whenever $y<z$, where $x$ here is $2r^2$, $y$ is $a_i^2+b_i^2$ and $z$ is $r\alpha^2+r$.
Thus we get
\begin{align*}
	\Delta_i &\ge \sqrt{\frac{r}{(\sqrt{2}+1)(\alpha^2+1)}+a_i^2}-a_i\ge^{(a)}\sqrt{\frac{r}{(\sqrt{2}+1)(\alpha^2+1)}+r\alpha^2}-\sqrt{r}\alpha\\
	&=\sqrt{r}\left[\sqrt{\frac{1}{(\sqrt{2}+1)(\alpha^2+1)}+\alpha^2}-\alpha\right]=\sqrt{r}\frac{\frac{1}{(\sqrt{2}+1)(\alpha^2+1)}}{\sqrt{\frac{1}{(\sqrt{2}+1)(\alpha^2+1)}+\alpha^2}+\alpha}\\
	&\ge\frac{\sqrt{r}}{(\sqrt{2}+1)(\alpha^2+1)^{3/2}}\coloneqq\Delta
\end{align*}
From the dynamics, $a_{i+1}=a_i+\tau(\bar{a}_{i+1}-a_i)\ge a_i + \tau\Delta_i\ge a_i + \tau\Delta=a_0+(i+1)\tau\Delta$.
Thus $a_T\ge T\tau\Delta$.
But we assumed that $a_T\le\alpha\sqrt{r}$, so we have 
\begin{align*}
	\sqrt{r}\alpha\ge T\tau\Delta\ge\frac{T\tau\sqrt{r}}{(\sqrt{2}+1)(\alpha^2+1)^{3/2}}\ge\frac{T\tau\sqrt{r}}{(\sqrt{2}+1)(\alpha^2+\alpha^2)^{3/2}}=\frac{T\tau\sqrt{r}}{(\sqrt{2}+4)\alpha^3}
\end{align*}
Thus we get that $\alpha>\frac{(T\tau)^{1/4}}{2}$.
This completes the proof
\end{proof}
\end{theorem}

We now prove why the initialization learned at the end of Reptile will help with sample complexity of new task.
We denote $\mX\sim\gN(0,I_d)^n$ as sampling $n$ i.i.d. vectors from $\gN(0,I_d)$ and stacking them into a matrix $\mX\in\R^{n\times d}$, and $\Sigma_\mX:=\frac{1}{n}\vx_i\vx_i^{\top}$
\begin{lemma}\label{lem:reg_closed_form}
Given a symmetric and invertible $\wf\in\R^{d\times d}$ as the first layer, the excess risk for learning the second layer is
\begin{align*}
	\gE_n(\gdtworeg(\cdot;(\wf,\vzero_d)),\mu_\wo)
	=&\ex_{\mX\sim\gN(0,I_d)^n} [\lambda^2\|\wf(\wf\Sigma_\mX \wf + \lambda I_d)^{-1}\wf^{-1}\wo\|^2]\\
	 +&\frac{\sigma^2}{n} \ex_{\mX\sim\gN(0,I_d)^n} \tr\left( \wf (\wf\Sigma_\mX \wf+\lambda I_d)^{-1} \wf \Sigma_\mX \wf (\wf\Sigma_\mX \wf+\lambda I_d)^{-1} \wf  \right)\end{align*}
\begin{proof}[Proof of \Lemref{lem:reg_closed_form}]
By definition, we have
\begin{align*}
	\gE_{n,\lambda}(\wf, \mu_{\wo}) &= \ex_{s\sim\{\pm1\}}\ex_{S\sim\rho_{s\wo}^n}\ell_{s\wo}(\gdtworeg(S;(\wf,\vzero_d)))-\sigma^2\\
	&=\ex_{s\sim\{\pm1\}}\ex_{S\sim\rho_{s\wo}^n}\|s\wo-\gdtworeg(S;(\wf,\vzero_d)))\|^2
\end{align*}
We first compute the inner expectation for $s=1$, a similar calculation will work for $s=-1$.
First, we state the solution for $\gd$ for the regularized loss $\ell_{S,\lambda}$ starting from $\wf$ and we prove this later.
Let $S=(\mX,\vy)$ be all the samples and predictions, where $\mX\in\R^{n\times d}$ and $\vy\in\R^n$.
Define $\xi=\vy-\mX^{\top}\wo$ to be the noise in the predictions; by the definition of $\rho_{\wo}$, we have that $\xi\sim\gN(0, \sigma^2 I_n)$.
We can now write the solution to $\gdtworeg$ by using \Lemref{lem:linear_dynamics} as
\begin{align*}
	\gdtworeg(S;(\wf,\vzero_d))=(\wf\Sigma_\mX \wf+\lambda I_d)^{-1}(\wf\Sigma_\mX\wo+\frac{1}{n}\wf\mX^{\top}\xi)
\end{align*}
The intuition is that $\ell_{S,\lambda}(\cdot,\wf)$ has a unique solution because of the regularization, and gradient descent converges to that unique solution.
Using this, we can compute the excess risk for $\rho_\wo$
\begin{align*}
	\ex_{S\sim\rho_{\wo}^n}\|\wo- \wf~\gdtworeg(S;(\wf,\vzero_d))\|^2 &= \ex_{\substack{\mX\sim\gN(0,I_d)^n\\\xi\sim\gN(0, \sigma I_n)}}\|\wo-\wf(\wf\Sigma_\mX \wf+\lambda I_d)^{-1}(\wf\Sigma_\mX\wo+\frac{1}{n}\wf\mX^{\top}\xi)\|^2\\
	& =  \ex_{\mX\sim\gN(0,I_d)^n}\|\wf\left(I_d-(\wf\Sigma_\mX \wf+\lambda I_d)^{-1}(\wf\Sigma_\mX \wf) \right) \wf^{-1}\wo\|^2\\
	&~~~~~~~~~~~~~~~~~~~~~~~~~~~+ \ex_{\substack{\mX\sim\gN(0,I_d)^n\\\xi\sim\gN(0, \sigma I_n)}}\| \wf(\wf\Sigma_\mX \wf+\lambda I_d)^{-1} \frac{1}{n}\wf \mX^\top \xi\|^2\\
	& =  \ex_{\mX\sim\gN(0,I_d)^n} \lambda^2\|\wf(\wf\Sigma_\mX \wf+\lambda I_d)^{-1} \wf^{-1}\wo\|^2\\
	&+ \frac{\sigma^2}{n} \ex_{\mX\sim\gN(0,I_d)^n} \tr\left( \wf (\wf\Sigma_\mX \wf+\lambda I_d)^{-1} \wf \Sigma_\mX \wf (\wf\Sigma_\mX \wf+\lambda I_d)^{-1} \wf  \right)
\end{align*}
\end{proof}
\end{lemma}
\begin{lemma}\label{lem:first_term}
Suppose $\wf=(\alpha-\kappa)\bwo\bwo^{\top}+\kappa I_d$, where $\alpha\ge\kappa$, then for $\alpha=\poly(\epsilon^{-1}, d, \kappa, \|\wo\|^2), \lambda=\Theta\left(\alpha^{3/2} \right)$ and $n=\Omega(\log(\epsilon^{-1}\|\wo\|_2))$, we have the following,
\begin{align*}
	\ex_{\mX\sim\gN(0,I_d)^n}[\lambda^2\|\wf(\wf\Sigma_\mX \wf + \lambda I_d)^{-1}\wf^{-1}\wo\|^2] \le \epsilon\\
	\ex_{\mX\sim\gN(0,I_d)^n} \tr\left( \wf (\wf\Sigma_\mX \wf+\lambda I_d)^{-1} \wf \Sigma_\mX \wf (\wf\Sigma_\mX \wf+\lambda I_d)^{-1} \wf  \right) \le 2 + \epsilon
\end{align*}
\end{lemma}

\begin{proof}
We write the SVD of $\wf$ as the following,

\begin{align*}
\wf = U
\begin{bmatrix}
\alpha &          &        &\\
		 & \kappa &        & \\
		 & 	        & \ddots & \\
		 &        
	     &        & \kappa\\
\end{bmatrix}
U^{\top} := U D_{\alpha, \kappa} U^{\top}
\end{align*}
where $D_{\alpha, \kappa} := \alpha \ve_1\ve_1^{\top} + \kappa (I_d - \ve_1\ve_1^{\top})$, and we know $U^{\top} \wo = \|\wo\|\ve_1$.

For simplicity, from now on we write $\Sigma:=U^{\top}\Sigma_\mX U$ which is identically distributed as $\Sigma_\mX$, and we let $v\in\R^{d}$ denote the top eigenvector of $D_{\alpha,\kappa} \Sigma D_{\alpha,\kappa}$. Now we use an eigenvector perturbation argument to show $v$ is close to $\ve_1$ if $\alpha$ is much larger than $\kappa$. For this purpose, we write  $D_{\alpha,\kappa} \Sigma D_{\alpha,\kappa}=\alpha^2\Sigma_{11}\ve_1\ve_1^{\top} + E$ where 
\begin{align*}
E:= \kappa D_{\alpha,\kappa} \Sigma (I_d - \ve_1\ve_1^{\top}) + \kappa (I_d - \ve_1\ve_1^{\top}) \Sigma D_{\alpha,\kappa} + \kappa^2 (I_d - \ve_1\ve_1^{\top})\Sigma (I_d - \ve_1\ve_1^{\top})
\end{align*}

It is clear that
\begin{align*}
\|E\|_F &\leq \left(2\kappa \|D_{\alpha,\kappa} \|_2 \|I_d - \ve_1\ve_1^{\top}\|_F + \kappa^2\|I_d - \ve_1\ve_1^{\top}\|_F^2 \right) \|\Sigma\|_F\\
&\leq \left(2\alpha d\kappa + d\kappa^2\right) \tr(\Sigma)
\end{align*}
By the Davis-Kahan theorem~\cite{davis1970rotation}, we have
\begin{align*}
\|vv^{\top} - \ve_1\ve_1^{\top}\|_F\leq 2\sqrt{2}\frac{\|E\|_F}{\alpha^2\Sigma_{11}} \leq C \frac{\left(2\alpha d\kappa + d\kappa^2\right) \tr(\Sigma)}{\alpha^2 \Sigma_{11}}
\end{align*}
where $C$ is an absolute constant. Furthermore, we can bound the eigenvalues of $D_{\alpha,\kappa} \Sigma D_{\alpha,\kappa}$ using Weyl's inequality:
\begin{align*}
\left| \sigma_{1}(D_{\alpha,\kappa} \Sigma D_{\alpha,\kappa}) - \alpha^2\Sigma_{11} \right| \leq \|E\|_F \leq \left(2\alpha d\kappa + d\kappa^2\right) \tr(\Sigma)\\
\forall i: 2\leq i\leq d,~~~\left| \sigma_{i}(D_{\alpha,\kappa} \Sigma D_{\alpha,\kappa})\right| \leq \|E\|_F \leq \left(2\alpha d\kappa + d\kappa^2\right) \tr(\Sigma)
\end{align*}
where $\sigma_1$ denotes the largest eigenvalue and $\sigma_i$'s are the rest. It follows that
\begin{align*}
\lambda^2 & \|\wf(\wf\Sigma_\mX \wf + \lambda I_d)^{-1}\wf^{-1}\wo\|^2 = \lambda^2 \|\wo\|^2 \|D_{\alpha,\kappa} (D_{\alpha,\kappa} \Sigma D_{\alpha,\kappa} +\lambda I_d)^{-1} D_{\alpha,\kappa}^{-1} \ve_1\|^2 \\
&\leq \frac{\lambda^2\|\wo\|^2}{(\alpha^2\Sigma_{11} -  \left(2\alpha d\kappa + d\kappa^2\right) \tr(\Sigma) + \lambda)^2} \|D_{\alpha,\kappa} vv^{\top} \frac{1}{\alpha}\ve_1\|^2 +  \|D_{\alpha,\kappa} (I_d - vv^{\top}) \frac{1}{\alpha}\ve_1\|^2 \|\wo\|^2 \\
&\leq \frac{\lambda^2}{(\alpha^2\Sigma_{11} -  \left(2\alpha d\kappa + d\kappa^2\right) \tr(\Sigma) + \lambda)^2}\|\wo\|^2  + \left(1 + \frac{d\kappa}{\alpha}\right)^2 \|vv^{\top} - \ve_1\ve_1^{\top}\|_F^2 \|\wo\|^2 \\
&\leq \frac{\lambda^2}{(\alpha^2\Sigma_{11} -  \left(2\alpha d\kappa + d\kappa^2\right) \tr(\Sigma) + \lambda)^2}\|\wo\|^2  + \left(1 + \frac{d\kappa}{\alpha}\right)^2 
\frac{\left(2\alpha d\kappa + d\kappa^2\right)^2 \tr(\Sigma)^2}{(\alpha^2 \Sigma_{11})^2}\|\wo\|^2
\end{align*}
where $C'$ is another absolute constant.

Finally, we note that $n \Sigma_{ii} \sim \chi^2(n)$, i.e., $\chi^2$ distribution with $n$ degree of freedom for all $i\in[d]$. Thus by standard concentration bound, we have
$\Pr\left[ \Sigma_{11} \geq 0.9 \wedge \tr(\Sigma)\leq 2d \right] \geq 1-\exp(-\Omega(n))$. To evaluate the expectations, we condition on two events, namely $\Sigma_{11} \geq 0.9 \wedge \tr(\Sigma)\leq 2d$ and its complement. Thus in the case where $\alpha=\Omega\left(\max\{ \epsilon^{-1} d^4\kappa^4 \|\wo\|^2\}\right)$, $\lambda = \Theta(\alpha^{3/2})$ and $n=\Omega\left( \log(\epsilon^{-1}\|\wo\|) \right)$, we have
\begin{align*}
&\ex_{\mX\sim\gN(0,I_d)^n}[\lambda^2\|\wf(\wf\Sigma_\mX \wf + \lambda I_d)^{-1}\wf^{-1}\wo\|^2] \\
&\leq \|\wo\|^2\exp(-\Omega(n)) \\
&~~~~~+ \left(1-\exp(-\Omega(n)\right) \|\wo\|^2 \left( \frac{\lambda^2}{(0.9\cdot\alpha^2 -  4\alpha d^2\kappa - 2d^2 \kappa^2 + \lambda)^2} + 4\left(1 + \frac{d\kappa}{\alpha}\right)^2\frac{\left(2\alpha d\kappa + d\kappa^2\right)^2 d^2}{(0.9\cdot\alpha^2)^2}\right)\\
&\leq\epsilon
\end{align*}

For the second part,  we have
\begin{align*}
&\ex_{\mX\sim\gN(0,I_d)^n}\| \wf(\wf\Sigma_\mX \wf+\lambda I_d)^{-1} \wf \mX^\top \|_F^2 \leq  \ex_{\mX\sim\gN(0,I_d)^n}[ \|\wf\|_2^2 ~\tr((\wf\Sigma_\mX \wf+\lambda I_d)^{-2} \wf \Sigma_\mX \wf)] \\
&= \ex_{\mX\sim\gN(0,I_d)^n}\left[\alpha^2 \sum_{i=1}^{d} \frac{\sigma_i(D_{\alpha, \kappa} \Sigma D_{\alpha, \kappa})}{ \left(\sigma_i(D_{\alpha, \kappa} \Sigma D_{\alpha, \kappa}) + \lambda \right)^2} \right]\\
&\leq \ex_{\mX\sim\gN(0,I_d)^n}\left[\frac{\sigma_1(D_{\alpha, \kappa} \Sigma D_{\alpha, \kappa})}{\sigma_1(D_{\alpha, \kappa} \Sigma D_{\alpha, \kappa})+\lambda} \cdot \frac{\alpha^2}{\sigma_1(D_{\alpha, \kappa} \Sigma D_{\alpha, \kappa})+\lambda} + (d-1)\frac{\alpha^2}{\lambda^2} \|E\|_F \right] \\
&\leq \ex_{\mX\sim\gN(0,I_d)^n}\left[\frac{\alpha^2}{\sigma_1(D_{\alpha, \kappa} \Sigma D_{\alpha, \kappa}) +\lambda} + (d-1)\frac{\alpha^2}{\lambda^2} \left(2\alpha d\kappa + d\kappa^2\right) \tr(\Sigma) \right] \\
&\leq \underbrace{\ex_{\mX\sim\gN(0,I_d)^n}\left[\frac{\alpha^2}{\alpha^2\Sigma_{11} - \left(2\alpha d\kappa + d\kappa^2\right) \tr(\Sigma) +\lambda} \right]}_{:=\diamondsuit} + \frac{\alpha^2}{\lambda^2}(2 d^3\kappa^2 + d^2\kappa^2)\\
\end{align*}

In order to bound $\diamondsuit$, we first condition on the event of $\mathcal{E}:= \Sigma_{11} \geq \frac{1}{\sqrt{\alpha}} \wedge \tr(\Sigma) \leq \frac{d\sqrt{\alpha}}{8} $ which occurs with overwhelming probability. In fact, we have by the standard concentration bound and the CDF of $\chi^2(n)$ distribution, i.e., $\Pr[\Sigma_{11}\leq \frac{1}{\sqrt{\alpha}}]\leq (1/\alpha)^{n/4}$ that
\begin{align*}
\Pr[\mathcal{E}]\geq 1 - \alpha^{-n/4} - \exp(-\Omega(\sqrt{\alpha})
\end{align*}
It follows that
\begin{align*}
\diamondsuit&\leq \ex_{\mX\sim\gN(0,I_d)^n}\left[\frac{\alpha^2}{\alpha^2\Sigma_{11} - \left(2\alpha d\kappa + d\kappa^2\right) \tr(\Sigma) +\lambda} \middle| \mathcal{E}\right] \Pr[\mathcal{E}] \\
&~~~~~~~~~~~~~~~~~~~~~+ \ex_{\mX\sim\gN(0,I_d)^n}\left[\frac{\alpha^2}{\alpha^2\Sigma_{11} - \left(2\alpha d\kappa + d\kappa^2\right) \tr(\Sigma) +\lambda} \middle| \neg\mathcal{E}\right] (1-\Pr[\mathcal{E}]) \\
&\leq \ex_{\mX\sim\gN(0,I_d)^n}\left[\frac{\alpha^2}{\alpha^2\Sigma_{11} - \left(2\alpha d\kappa + d\kappa^2\right) \tr(\Sigma) +\lambda} \middle| \mathcal{E}\right] \Pr[\mathcal{E}] + \alpha^2 (1-\Pr[\mathcal{E}]) \\
&\leq \ex_{\mX\sim\gN(0,I_d)^n}\left[\frac{\alpha^2}{\alpha^2\Sigma_{11} - \left(2\alpha d\kappa + d\kappa^2\right) \tr(\Sigma) +\lambda} \middle| \mathcal{E}  \right] + \alpha^2\left(\alpha^{-n/4}+\exp(-\Omega(\sqrt{\alpha}))\right) \\
&\leq  \ex_{\mX\sim\gN(0,I_d)^n}\left[\frac{\alpha^2}{0.5\cdot\alpha^2\Sigma_{11} +\lambda} \middle| \mathcal{E}    \right] \left(1-\alpha^{-n/4} - \exp(-\Omega(\sqrt{\alpha}))\right) + \alpha^2\left(\alpha^{-n/4} + \exp(-\Omega(\sqrt{\alpha}))\right) \\
&\leq  \ex_{\mX\sim\gN(0,I_d)^n}\left[\frac{\alpha^2}{0.5\cdot\alpha^2\Sigma_{11} +\lambda} \middle| \mathcal{E} \right] + \alpha^2\left(\alpha^{-n/4} + \exp(-\Omega(\sqrt{\alpha}))\right) \\
&\leq \ex_{\mX\sim\gN(0,I_d)^n}\left[\frac{2}{\Sigma_{11}}\right] + \alpha^2\left(\alpha^{-n/4} + \exp(-\Omega(\sqrt{\alpha}))\right)\\
&\leq \frac{2 n }{n-2} +  \alpha^2\left(\alpha^{-n/4} + \exp(-\Omega(\sqrt{\alpha}))\right) \end{align*}
where we use the fact that $\ex[2/\Sigma_{11}| \mathcal{E}]\leq \ex[2/\Sigma_{11}]$ and the expectation of inverse $\chi^2$ distribution.
Putting it together and assuming $\alpha=\Omega\left(\poly(\epsilon^{-1}d^3\kappa^2)\right), \lambda=\alpha^{3/2}$ and $n\geq 10$, we conclude 
\begin{align*}
&\ex_{\mX\sim\gN(0,I_d)^n}\| \wf(\wf\Sigma_\mX \wf+\lambda I_d)^{-1} \wf \mX^\top \|_F^2 \\
&\leq \frac{2 n }{n-2} +  \alpha^2\left(\alpha^{-n/4} + \exp(-\Omega(\sqrt{\alpha}))\right) + \frac{\alpha^2}{\lambda^2}(2 d^3\kappa^2 + d^2\kappa^2)\\
&\leq 2 + \epsilon
\end{align*}
\end{proof}
\subsection{Proof of Main Results}\label{asubsec:main_results}

\paragraph{Reptile:} We finally prove the main theorem about the success of \reptile.
\begin{lemmano}[{\bf \Thmref{thm:upper_bound_reptile}}]
Starting with $(\wf_0,\ws_0)=(\kappa I_d,\vzero_d)$, let $\wf_T=\reptile(\rho_{1:T},(\wf_0,\ws_0))$ be the initialization learned using $T$ tasks $\{\rho_{1},\dots,\rho_{T}\}\sim_{i.i.d.} \mu_{\wo}^T$. If $T\ge poly(d,r,1/\epsilon,\log(1/\delta),\kappa)$ and $\tau=\gO(T^{-1/3})$, then with probability at least $1-\delta$ over sampling of $T$ tasks,
\begin{align*}
	\min\limits_{\lambda\ge0}~\gE_{n}(\gdtworeg(\cdot;(\wf_T,\vzero_d)), \mu_{\wo}) \le \epsilon + \frac{cr^2}{n}
\end{align*}
for a small constant $c$. Thus with the same probability, we have
\begin{align*}
	\min\limits_{\lambda\ge 0}~n_\epsilon&(\gdtworeg(\cdot;\wf_T,\vzero_d), \mu_\wo) = \gO\left(\frac{r^2}{\epsilon}\right)
\end{align*}
\end{lemmano}
\begin{proof}[Proof of \Thmref{thm:upper_bound_reptile}]
The theorem essentially follows from \Lemref{lem:reptile_dynamics}, \Thmref{thm:ablowup}, \Lemref{lem:reg_closed_form} and \Lemref{lem:first_term}.
From \Lemref{lem:reptile_dynamics} and \Thmref{thm:ablowup}, we get that with probability at least $1-\delta$ choosing $\tau=T^{-1/3}\log(2T/\delta)^{-2/3}$ will ensure $\wf_T=(\alpha-\kappa)\bwo{\bwo}+\kappa I_d$ with $\alpha={\Omega}(\sqrt{r}T^{1/6})$.
Combining \Lemref{lem:reg_closed_form} and \Lemref{lem:first_term} we know that if $\alpha=\Omega(\poly(\epsilon^{-1},d,\kappa,r))$, then $\gE_n(\gdtworeg(\cdot;(\wf,\vzero_d)),\mu_\wo) \le \frac{\epsilon}{2} + \frac{c\sigma^2}{n} = \frac{\epsilon}{2} + \frac{cr^2}{n}$.
To ensure $\alpha$ is this large, we just need that the number of tasks to satisfy $T=poly(\epsilon^{-1},d,\kappa,r,\log(\delta^{-1}))$ for the appropriate polynomial from \Lemref{lem:first_term}.
Thus for $\gE_n(\gdtworeg(\cdot;(\wf,\vzero_d)),\mu_\wo) \le \epsilon$, we just need $n=\Omega\left(\frac{r^2}{\epsilon}\right)$ samples for a new task, completing the proof.
\end{proof}

\paragraph{Representation learning:} We now prove the main theorem about the success of \replearn.
\begin{lemmano}[{\bf \Thmref{thm:upper_bound_replearn}}]
Starting with $(\wf_{0},\ws_{0,1:T})=(\kappa I_d,\vzero_d,\dots,\vzero_d)$, let $\wf_T=\replearn(\rho_{1:T},(\wf_{0},\ws_{0,1:T})),$ be the initialization learned using $T$ tasks $\{\rho_{1},\dots,\rho_{T}\}\sim_{i.i.d.} \mu_{\wo}^T$. If $T\ge poly(d,r,1/\epsilon,\log(1/\delta),\kappa)$, then with probability at least $1-\delta$ over sampling of the $T$ tasks,
\begin{align*}
	\min\limits_{\lambda\ge0}~\gE_{n}(\gdtworeg(\cdot;(\wf_T,\vzero_d)), \mu_{\wo}) \le \epsilon + \frac{cr^2}{n}
\end{align*}
for a small constant $c$. Thus with the same probability, we have
\begin{align*}
	\min\limits_{\lambda\ge 0}~n_\epsilon&(\gdtworeg(\cdot;\wf_T,\vzero_d), \mu_\wo) = \gO\left(\frac{r^2}{\epsilon}\right)
\end{align*}
\end{lemmano}
\begin{proof}[Proof of \Thmref{thm:upper_bound_replearn}]
The proof of this is very similar to the proof of \Thmref{thm:upper_bound_reptile} above.
Just as in that proof, we need to show that for a large enough $T$, $\wf_T\coloneqq\wf_T^{\replearn}=(\alpha-\kappa)\wo\wo^{\top} + \kappa I_d$ for a large enough $\alpha$.
The theorem will then follow from \Lemref{lem:reg_closed_form} and \Lemref{lem:first_term} just as in the previous proof.
To prove the closed form solution for $\wf_T$, we use the following lemma that is very similar to \Lemref{lem:gd_dynamics}
\begin{lemma}\label{lem:gd_dynamics_replearn}
Starting from $\wf(0)=(a(0)-\kappa)\bwo\bwo^{\top} + \kappa I_d$, $\ws_i(0)=\vzero_d,~i\in[T]$, with $a(0)>0$, the solution of gradient flow on loss $\gL_{rep}(\wf,\ws_{1:T})$ for $s\in\{\pm1\}$, is $\bar{\wf},\bar{\ws}_{1:T}$, where
\begin{align*}
	\bar{\wf}=(\bar{a}-\kappa)\bwo\bwo^{\top} &+ \kappa I_d, ~\bar{\ws}_i=\bar{b}_i\bwo~\text{, where}\\
	\bar{a} = \sqrt{\frac{a(0)^2+\sqrt{4r^2T+a(0)^4}}{2}}, ~\bar{b}_i &= s_i\sqrt{\frac{-a(0)^2+\sqrt{4r^2T+a(0)^4}}{2}}
\end{align*}
\end{lemma}
\begin{proof}
We first rewrite the representation learning objective using the derivation in \Secref{subsec:replearn_sketch} as follows
\begin{align}\label{eq:replearn_gd_dynamics}
	\gL_{rep}(\wf,\ws_{1:T}) = \frac{1}{T}\|\wf^{\top}\Ws-\Wo\|^2
\end{align}
where $\Ws\in\R^{d\times T}, \Wo\in\R^{d\times T}$ and the $i^{th}$ column of $\Ws$ is $\ws_i$ and the $i^{th}$ column of $\Wo$ is $s_i\wo$.
Just as in \Lemref{lem:gd_dynamics}, we define $\mU$ to be an orthogonal matrix whose first column is $\bwo$.
We also define $\mV\in\R^T$ to be the vector of the signs of the tasks, i.e. $\mV=\frac{1}{\sqrt{T}}(s_1,\dots,s_T)$.
We can then rewrite $\wf(0) = \mU\Lambda_1(0)\mU^{\top}$, where $\Lambda_1(0)\in\R^{d\times d}$ is a diagonal matrix that looks like $\Lambda_1(0) = \diag(a(0),\kappa,\dots,\kappa)$, $\Ws(0)=\mU\Lambda_2(0)\mV^{\top}$, where $\Lambda_2(0)=(b(0),0,\dots,0)\in\R^d$ with $b(0)=0$ and $\Wo=\mU\Lambda_*\mV^{\top}$, where $\Lambda_*=(\sqrt{T}r,0,\dots,0)\in\R^d$.
Note that $\mU^{\top}\mU = I_d$ and $\mV^{\top}\mV=1$

The dynamics of gradient flow on $\gL_{rep}$ using $\Eqref{eq:replearn_gd_dynamics}$ is
\begin{align*}
	\frac{d\wf(t)}{dt}&=\Ws(t)\Wo^{\top}-\Ws(t)\Ws(t)^{\top}\wf(t)\\
	\frac{d\Ws(t)}{dt}&=\wf(t)\Wo-\wf(t)\wf(t)^{\top}\Ws(t)
\end{align*}
By defining $\Lambda_1(t)=\mU^{\top}\wf(t)\mU$, $\Lambda_2(t)=\mU^{\top}\Ws(t)\mV$, $\Lambda_*=\mU^{\top}\Wo\mV$, we can multiply the above dynamics by $\mU^{\top}$ on the left and $\mV$ on the right, and use the properties above to get
\begin{align*}
	\frac{d\Lambda_1(t)}{dt}&=\Lambda_2(t)\Lambda_*^{\top}-\Lambda_2(t)\Lambda_2(t)^{\top}\Lambda_1(t)\\
	\frac{d\Lambda_2(t)}{dt}&=\Lambda_1(t)\Lambda_*^{\top}-\Lambda_1(t)\Lambda_1(t)^{\top}\Lambda_2(t)
\end{align*}
Just like \Lemref{lem:gd_dynamics}, this reduces to a scalar dynamics and the solution we get is $\bar{\wf}=\mU\bar{\Lambda}_1\mU^{\top}$, $\bar{\Ws}=\mU\bar{\Lambda}_2\mV^{\top}$, where $\bar{\Lambda}_1=diag(\bar{a},\kappa,\dots,\kappa), \bar{\Lambda}_2=(\bar{b},0,\dots,0)$ and 
\begin{align*}
	\bar{a}=\sqrt{\frac{a(0)^2+\sqrt{4r^2T+a(0)^4}}{2}},~~
	\bar{b}=\sqrt{\frac{-a(0)^2+\sqrt{4r^2T+a(0)^4}}{2}}
\end{align*}
This completes the proof of the lemma.
\end{proof}
Back to the main theorem, we see from the above lemma that $\alpha=\Omega(\sqrt{r}T^{1/4})$, where $\wf_T=(\alpha-\kappa)\wo\wo^{\top} + \kappa I_d$.
So making $T=poly(\epsilon^{-1},d,\kappa,r,\log(\delta^{-1}))$ large enough will make $\alpha$ large enough to invoke \Lemref{lem:reg_closed_form} and \Lemref{lem:first_term} to complete the proof, just like in the proof of \Thmref{thm:upper_bound_reptile}.
\end{proof}

\section{Information-Theoretic Lower-Bounds for the Convex Case}\label{asec:gv_lower}

\begin{theorem}\label{app:thm:gvlower}
	For any $G,V>0$, there exists a domain $\gZ$, parameter class $\Theta\subseteq\R^d$ and a distribution $\mu$ over tasks such every $\rho\sim\mu$ is a distribution over $\gZ$ and $\ell_{\rho}(\theta)=\E_{z\sim \rho}\ell_z(\theta)$ where $\ell_z:\Theta\rightarrow\R$ is convex and $G$-Lipschitz  w.r.t. the Euclidean norm for every $z\in\gZ$. Additionally, $\Theta$ satisfies
	$$\min_{\phi\in\Theta}\E_{\rho\sim\mu}\|\phi-\operatorname{Proj}_{\Theta_\rho^\ast}(\phi)\|\le V$$
	and
	$$\gE_n(\Alg, \mu)=\Omega\left(GV\min\left\{\frac1{\sqrt n},\frac1{\sqrt d}\right\}\right)$$
	for any algorithm $\Alg:\gZ^n\rightarrow\Theta$ that returns a parameter given a training set.
\end{theorem}
\begin{proof}
	This result extends the result of \citet[Theorem~1]{agarwal:12} to the case of distributions over functions;
	all equations and statements referenced in this proof are from that paper.
	We first define the domain $\gZ$, parameter class $\Theta$, meta-distribution $\mu$ and the within-task distributions and losses.
	
	\textbf{Parameter class:} We use a $\ell_2$ ball of radius $V$ as the class, i.e. $\Theta=\{\theta\in\R^d: \|\theta\|\le \nicefrac{V}{2}\}$.
	
	\textbf{Domain and loss:} We defined $\gZ$ to be a tuple of an index and a bit, i.e. $\gZ=[d]\times\{0,1\}$.
	For a given $z\in\gZ$, we define $\ell_z$ as follows
	\begin{align*}
		\ell_z(\theta) = \begin{cases}
		G\left|\theta(i) + \frac{V}{2\sqrt{d}}\right| & \text{if $z=(i,1), i\in[d]$}\\\\
		G\left|\theta(i) - \frac{V}{2\sqrt{d}}\right| & \text{if $z=(i,0), i\in[d]$}
		\end{cases}
	\end{align*}
	Note that $\ell_z$ is convex and $G$-Lipschitz for every $z\in\gZ$.
	
	\textbf{Meta-learning distribution:} We define the distribution $\mu$ on the vertices of the hypercube $\{\pm 1\}^d$.
	First we let $\gV$ be the $\frac{d}{4}$-packing of the hypercube  in the Hamming distance defined in \citet{agarwal:12}.
	Each task $\rho_\alpha$ is parametrized by a vertex $\alpha\in\gV$.
	To sample a new task $\rho_\alpha\sim\mu$, we sample $\alpha\sim\gV$ uniformly and return $\rho_\alpha$ that we define below.

	\textbf{Data distribution:} For a given task $\rho_\alpha\sim\mu$, we define a distribution over $\gZ$.
	Sampling $z\sim\rho_\alpha$ is equivalent to first sample an index uniformly at random, $i\sim[d]$, and then independently sampling a bit from a biased Bernoulli distribution $b\sim\operatorname{Ber}\left(\frac{1}{2}+\alpha(i)\delta\right)$, for some $\delta\in(0,\nicefrac{1}{4})$, and returning $(i,b)$.
	Thus the population loss for $\rho_\alpha$ becomes
	\begin{align*}
		\ell_{\rho_\alpha}(\theta) = \sum\limits_{i=1}^d \left(\frac{1}{2}+\alpha(i)\delta\right)\left|\theta(i) + \frac{V}{2\sqrt{d}}\right| + \left(\frac{1}{2}-\alpha(i)\delta\right)\left|\theta(i) + \frac{V}{2\sqrt{d}}\right|
	\end{align*}
	It is not difficult to see that the minimizer of the population loss $\theta^*_{\rho_\alpha}\in\R^d$ in fact lies in $\Theta$ and is
	\begin{align*}
		\theta^*_{\rho_\alpha}(i) = \begin{cases}
			-\frac{V}{2\sqrt{d}} & \text{if $\alpha(i) = 1$}\\
			\frac{V}{2\sqrt{d}} & \text{if $\alpha(i) = -1$}
		\end{cases}
	\end{align*}
	Crucially, we note that since $\theta^*_{\rho_\alpha}\in\Theta$ for every $\alpha\in\gV$, we get that 
	\begin{align*}
		\min_{\phi\in\Theta}\E_{\rho\sim\mu}\|\phi-\operatorname{Proj}_{\Theta_\rho^\ast}(\phi)\|
		\le \E_{\rho\sim\mu}\|\vzero_d-\theta^*_{\rho_\alpha}\| = V
	\end{align*}
	Given this setup, we are ready to prove a lower bound for $\gE_n(\Alg,\mu)$ using the result from \citet{agarwal:12}.
	We define the class of functions $\gG(\delta)=\{\ell_{\rho_\alpha}: \alpha\in\gV\}$ and define $g_\alpha=\ell_{\rho_\alpha}$. Note that this is the same definition of $\gG(\delta)$ as in \citet{agarwal:12}.

	We now follow their proof of Theorem~1, where in addition to the randomness of sampling from the task-distribution $\rho_\alpha$ we must consider the randomness of sampling $\alpha\sim\gV$.
	This manifests only in the application of Lemmas~2 and~3 from their paper.
	We can modify their proof of Lemma 2 to only assume
	\begin{align*}
		\gE_n(\Alg,\mu)=\ex_{\alpha\sim\gV}[\Delta(\Alg,\alpha)]\le\frac{\psi(\delta)}{9}\text{, where  } \Delta(\Alg,\alpha) = \ex_{S\sim\rho_\alpha^n}[\ell_{\rho_\alpha}(\Alg(S))-\ell^*_{\rho_\alpha}]
	\end{align*}	
	instead of Equation~21 which effectively assumes $\max_{\alpha\in\gV}\Delta(\Alg,\alpha)\le\frac{\psi(\delta)}{9}$, where $\psi(\delta)$ is defined in Equation~19.
	We can modify the application of Markov's inequality, to get
	\begin{align*}
		\ex_{\alpha\sim\gV}\mathbb{P}_{S\sim\rho_\alpha^n}(\Alg(S)\ne\alpha)\le\ex_{\alpha\sim\gV}\mathbb{P}_{S\sim\rho_\alpha^n}(\Delta(\Alg,\alpha)\ge\psi(\delta)/3)\le\ex_{S\sim\rho_\alpha^n}\frac{\Delta(\Alg,\alpha)}{\nicefrac{\psi(\delta)}{3}}\le\frac{\nicefrac{\psi(\delta)}{9}}{\nicefrac{\psi(\delta)}{3}}\le \nicefrac{1}{3}
	\end{align*}
	where the first step is the same as in their proof, second step from Markov's inequality and third is from the assumption.
	The main difference from their proof, just like the assumption, is that we take expectation over $\alpha\in\gV$ rather than a maximum.

	For Lemma~3, note that the result already includes the randomness of sampling $\alpha\sim\operatorname{Unif}(\gV)$.
	Applying these results in the proof of Theorem~1, we use $\delta=\frac{36\varepsilon\sqrt d}{GV}$ for target error $\varepsilon$ to obtain $n=\Omega(G^2V^2/\varepsilon^2)$ for all $d\ge11$ and $\varepsilon\le\frac{GV}{144\sqrt d}$, completing the proof.
\end{proof}  
\end{document}